%% file: iclr2027_conference.tex
\documentclass{article} % For LaTeX2e
\PassOptionsToPackage{table}{xcolor}
\usepackage{iclr2027_conference,times}

\input{math_commands.tex}

\usepackage{url}
\usepackage[normalem]{ulem}
\usepackage{amsmath}
\usepackage{amssymb}
\usepackage{amsfonts}
\usepackage{bbm}
\usepackage{mathtools}
\usepackage{amsthm}
\usepackage{graphicx}
\usepackage{booktabs}
\usepackage{multirow}
\usepackage{adjustbox}
\usepackage{array}
\usepackage{arydshln}
\usepackage{enumitem}
\usepackage{tcolorbox}
\tcbuselibrary{breakable}
\usepackage{capt-of}
\usepackage{pifont}
\usepackage{algorithm}
\usepackage{algpseudocode}
\usepackage{float}
\usepackage{wrapfig}
\makeatletter
\renewcommand{\theHALG@line}{\thealgorithm.\arabic{ALG@line}}
\makeatother

\newcommand{\ourmethod}{\textsc{POrchestra}}
\newcommand{\aorchestra}{\textsc{AOrchestra}}
\newcommand{\modelname}[1]{\texttt{#1}}
\definecolor{arporow}{HTML}{EAF3FB}
\definecolor{myblue}{HTML}{3B75C3}
\definecolor{mypurple}{HTML}{6A3D9A}
\definecolor{questionbg}{HTML}{EAF3FF}
\definecolor{questionborder}{HTML}{555555}
\definecolor{questionbrain}{HTML}{EFA09B}
\definecolor{questionface}{HTML}{F6C8A8}
\newcolumntype{L}[1]{>{\raggedright\arraybackslash}p{#1}}
\newcolumntype{C}[1]{>{\centering\arraybackslash}p{#1}}

\newtheorem{proposition}{Proposition}

\usepackage{hyperref}
\hypersetup{
    pdftitle={Proactive Agentic Orchestration},
    pdfauthor={Xueyang Feng, Chaozhuo Li, Xiaohe Bo, Jiakai Tang, Chenxi Li, Qiwei Ye, Xu Chen},
    colorlinks=true,
    linkcolor=mypurple,
    citecolor=myblue,
    urlcolor=mypurple,
}
\usepackage[nameinlink,capitalize,noabbrev]{cleveref}

\title{Proactive Agentic Orchestration}

\author{
{\small\bfseries Xueyang Feng\textsuperscript{1}\thanks{Work done during an internship at BAAI.\quad\textsuperscript{\ensuremath{\dagger}}Corresponding author.}, Chaozhuo Li\textsuperscript{4}, Xiaohe Bo\textsuperscript{1}, Jiakai Tang\textsuperscript{1}, Chenxi Li\textsuperscript{3}, Qiwei Ye\textsuperscript{4}, Xu Chen\textsuperscript{1}\textsuperscript{\ensuremath{\dagger}}}\\[0.5em]
\multicolumn{1}{c}{\small\normalfont\textsuperscript{1}\,Gaoling School of Artificial Intelligence, Renmin University of China}\\
{\small\textsuperscript{3}\,The Chinese University of Hong Kong Shenzhen\quad
\textsuperscript{4}\,Beijing Academy of Artificial Intelligence}\\
\multicolumn{1}{c}{\small\normalfont\texttt{\{xueyangfeng, xu.chen\}@ruc.edu.cn}}
}

\iclrfinalcopy % Show authors using the standard ICLR title layout.
\begin{document}
% Hide main-paper entries in the appendix-only contents.
\addtocontents{toc}{\protect\setcounter{tocdepth}{-1}}

\maketitle
\lhead{Preprint} % Neutral header for the arXiv preprint.

\vspace{-2em} % Reduce the gap between the author block and abstract.
\begin{abstract}
Agentic orchestration delegates tasks to sub-agents with independent contexts, providing a key path to extending the effective execution horizon of large language model agents. However, existing orchestration methods face two limitations: in execution, current methods fix sub-agents' goals and harnesses at creation time, limiting their adaptability to dynamic real-world environments; in optimization, sparse and delayed feedback constrains effective learning of the orchestration policy. To address these limitations, we introduce \emph{\ourmethod{}}, an on-the-fly proactive agentic orchestration framework. It allows sub-agents to proactively communicate with the orchestrator at any logical point to request harness adjustments while maintaining continuity of local execution. This mechanism not only enables sub-agents to adapt seamlessly to dynamic environments but also naturally inspires \emph{On-Policy Orchestration Distillation} (OPOD) based on hindsight feedback. OPOD attributes sub-agents' proactive revision requests back to preceding orchestration decisions to derive targeted improvement signals, and leverages these signals as process supervision alongside the final reward to optimize the orchestration policy. We evaluate \ourmethod{} on three long-horizon agent benchmarks spanning deep research, dynamic digital assistance, and software engineering. \ourmethod{} improves task success by an average of 9.1\% relative to the strongest baseline; combined with OPOD, it improves small-orchestrator success by an average of 10.0\% relative to existing training baselines. \ourmethod{} also achieves a more favorable cost--performance trade-off. Our code is publicly available at \url{https://github.com/XueyangFeng/POrchestra}.
\end{abstract}
% Abstract averages use reported table values and equal weight per comparison.
% Execution: mean of 100*(P/best_baseline-1) across 3 benchmarks x 3 backbones = 9.0543%.
% P: 71.5,47.7,71.0,78.2,53.9,77.0,77.6,55.5,79.0.
% Best baseline: 63.6,39.8,70.0,76.4,46.9,76.0,75.2,44.5,78.0.
% Optimization: mean of 100*(OPOD/SFT-1) on GAIA and GAIA2 = 9.9627%.
% OPOD: 77.9,49.2; SFT (POrchestra): 72.4,43.8.

\section{Introduction}

To tackle complex, long-horizon tasks, humans form organizations to orchestrate individual efforts and achieve what no one could accomplish alone~\citep{march1993organizations,galbraith1974organization}. Inspired by this, agentic orchestration~\citep{ruan2026aorchestra,zhang2026orchestrao1omnimodalagentorchestration,wang2026mas} has emerged as a key means of enabling large language models (LLMs)~\citep{qwen3.5,grattafiori2024llama3herdmodels} to tackle long-horizon tasks across deep research, digital assistance, and software engineering~\citep{mialon2024gaia,froger2026gaia,jimenez2024swebench}. In these settings, the orchestrator delegates subtasks to sub-agents with execution harnesses~\citep{wang2024survey} and independent contexts, thereby mitigating context rot~\citep{du-etal-2025-context} and enabling longer-horizon execution.

For real-world long-horizon tasks, agentic orchestration has progressed from static design toward runtime adaptation. \textbf{(i) Manual configuration.} Early frameworks, such as MetaGPT and AgentVerse~\citep{hong2024metagpt,chen2024agentverse}, rely on humans to predefine agent roles and collaboration structures. \textbf{(ii) Offline orchestration.} Subsequent approaches use an orchestrator to automatically generate agent harnesses and collaboration workflows for each task~\citep{zhang2025gdesigner,wang2025scoreflow}. However, interactions between agents and real-world environments are dynamic and error-prone (e.g., tool failures and shifts in task goals as the environment changes), and such unexpected disruptions can cause predefined harnesses to break down (\emph{harness mismatch})~\citep{cemri2025multi}. \textbf{(iii) Re-orchestration.} To enable runtime adaptation, recent work integrates harness orchestration into task execution. For example, \aorchestra{}~\citep{ruan2026aorchestra} re-creates subsequent sub-agents after preceding sub-agents finish execution, whereas MAS$^2$~\citep{wang2026mas} triggers re-orchestration when specific conditions are met, such as code execution errors.

Despite this progress, current re-orchestration methods face two limitations: \ding{182}\ \textit{Delayed Adaptation in Execution:} harness mismatches often emerge naturally during sub-agent execution, yet the orchestrator cannot readily detect or address them until execution ends or predefined conditions are triggered (Figure~\ref{fig:intro1}(b)). Consequently, sub-agents may continue operating under mismatched harnesses or pursuing outdated subgoals, wasting execution steps, incurring additional costs, and propagating local errors that disrupt the task. \ding{183}\ \textit{Sparse Supervision in Optimization:} the reward used to optimize the orchestrator is typically available only after multiple rounds of orchestration and the sub-agent executions end. Even then, this reward is difficult to attribute to individual orchestration decisions and offers little guidance for improving them.
In light of these limitations, a central question arises: \uline{Can we \textbf{seamlessly integrate dynamic adjustment} of sub-agent harnesses into execution while providing \textbf{denser, richer supervision} for orchestrator policy optimization?}

\begin{figure}[t]
    \centering
            \vspace{-2em}
    \includegraphics[width=\textwidth]{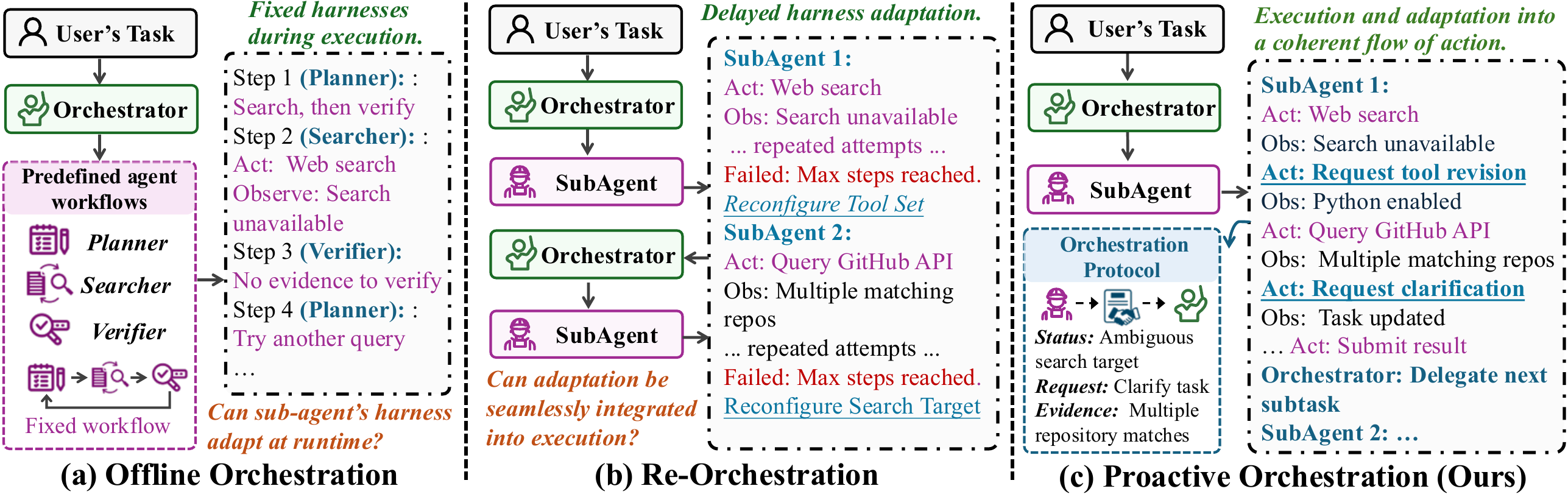}
        \vspace{-1.7em}
    \caption{Comparison of orchestration paradigms. \textbf{(a) Offline orchestration} fixes harnesses and workflows. \textbf{(b) Re-orchestration} adapts harnesses upon completion or predefined triggers. \textbf{(c) \ourmethod{}} integrates execution and harness adaptation via proactive sub-agent communication.}
    \vspace{-1.5em}
    \label{fig:intro1}
\end{figure}

To answer this question, we propose \emph{\ourmethod{}}, a proactive agentic orchestration paradigm. Our motivation is that, as task executors, sub-agents can recognize harness mismatches from local execution difficulties and request revisions. Accordingly, \ourmethod{} makes harness adjustment a collaborative process between sub-agents and the orchestrator. A sub-agent can interact with the environment or proactively pause execution at any logical point to send structured feedback containing revision suggestions and supporting evidence from its local execution perspective. The orchestrator updates the harness from a global task perspective. The update is recorded as an observation in the sub-agent's local trajectory, forming a continuous action stream interleaving harness adaptation with task execution. In this way, the orchestrator is no longer limited to creating sub-agents with fixed harnesses; it can revise their harnesses throughout their lifecycles based on feedback.

%Furthermore, proactive sub-agent feedback reveals how a harness performs in practice and where it falls short. Such feedback not only supports on-the-fly adjustment during execution but also indicates how preceding orchestration decisions should be revised. To exploit this process feedback, we introduce \emph{On-Policy Orchestration Distillation} (OPOD). OPOD extracts proactive sub-agent feedback from trajectories generated by the current orchestration policy and summarizes the exposed configuration mismatches and suggested corrections as natural-language gradients. It then appends these gradients to the original decision context to construct a feedback-augmented self-teacher using the same model, and distills the self-teacher's improved decision distribution into the original orchestration policy. Meanwhile, OPOD retains task-level outcome supervision through reward-weighted regression, combining sparse terminal rewards with targeted process feedback to optimize the orchestration policy more effectively. To summarize, our contributions are as follows:

Furthermore, proactive sub-agent feedback reveals obstacles encountered during execution and provides corresponding suggestions for improvement. This feedback mechanism not only enables sub-agents to adapt seamlessly to dynamic environments during execution but also naturally motivates \emph{On-Policy Orchestration Distillation} (OPOD), which optimizes the orchestrator using hindsight feedback. OPOD first attributes execution feedback to prior decisions and distills it into natural-language improvement gradients. It then appends these gradients to the original decision context and uses the same model to construct a feedback-augmented self-teacher, providing process-level supervision for orchestration decisions through self-distillation. Meanwhile, OPOD retains task-level outcome supervision via reward-weighted regression, combining sparse final rewards with targeted process feedback to optimize the orchestration policy. Our contributions are as follows:

%We evaluate \ourmethod{} on three representative long-horizon agent benchmarks covering deep research (GAIA~\citep{mialon2024gaia}), digital assistance in dynamic and asynchronous environments (GAIA2~\citep{froger2026gaia}), and software engineering (SWE-Bench~\citep{jimenez2024swebench}). Experiments show that \ourmethod{} not only improves runtime task performance, but also effectively supports the training of smaller orchestrators.

\begin{itemize}[leftmargin=1.5em, itemsep=0pt, topsep=0pt]
    \item[\ding{182}] \textbf{Execution Perspective:} We introduce \emph{\ourmethod{}}, an agentic orchestration paradigm that enables sub-agents to proactively participate in their own orchestration, allowing their harnesses to adapt seamlessly as dynamic, long-horizon tasks unfold in real-world environments.
    \item[\ding{183}] \textbf{Optimization Perspective:} We introduce OPOD, an on-policy orchestration self-distillation method that converts proactive feedback from sub-agents in \ourmethod{} into process-level supervision, enabling more targeted optimization of the orchestrator policy.
    \item[\ding{184}] \textbf{Empirical Evaluation:} We evaluate \ourmethod{} on three long-horizon benchmarks spanning deep research, dynamic digital assistance, and software engineering. Results demonstrate \textbf{(i) superior orchestration}, improving task success by an average of 9.1\% relative to the strongest baseline; \textbf{(ii) enhanced learnability}, with \ourmethod{} and OPOD improving small-orchestrator success by an average of 10.0\% relative to training baselines; and \textbf{(iii) Pareto-optimal efficiency}, achieving the highest performance along the cost--performance frontier.
\end{itemize}

\section{Related Work}
\label{sec:related-work}

\paragraph{Agentic Orchestration Systems.}
Agentic orchestration decomposes complex tasks by creating sub-agents and coordinating their work. In early \textbf{manually configured agent systems}, such as AutoGen, MetaGPT, and AgentVerse, developers predefine sub-agent composition and collaboration~\citep{wu2024autogen,hong2024metagpt,chen2024agentverse}; \textbf{offline orchestration} automatically generates agent systems before execution: G-Designer learns communication topologies, while ScoreFlow searches workflows using task-level feedback~\citep{zhang2025gdesigner,wang2025scoreflow}. Both paradigms fix agent configurations before task execution. \textbf{Re-orchestration} further enables runtime adaptation: MAS$^2$ revises generated workflows when specific errors occur~\citep{wang2026mas}; methods such as Orchestra-o1 and \aorchestra{} use sub-agents as the unit of orchestration, re-creating subsequent sub-agents after preceding sub-agents finish execution~\citep{zhang2026orchestrao1omnimodalagentorchestration,ruan2026aorchestra}. This sub-agent-based paradigm has demonstrated practical value in real-world applications~\citep{openai2026codexsubagents,anthropic2026claudecodesubagents}. In \ourmethod{}, we follow the same paradigm but introduce proactive sub-agents to seamlessly integrate orchestration into their execution.

\paragraph{Orchestration Policy Optimization.}
%Orchestration policies learn from demonstrations or task outcomes: \aorchestra{} uses supervised fine-tuning~\citep{ruan2026aorchestra}, while MAS-Orchestra uses sparse, delayed task rewards for RL~\citep{ke2026masorchestra}. For local supervision, Uno-Orchestra augments terminal rewards with format shaping and turn-level credit assignment~\citep{cui2026unoorchestra}; DA-GRPO scores candidate decisions with an LLM rubric~\citep{zhang2026orchestrao1omnimodalagentorchestration}. These scores offer limited attribution and revision guidance for execution obstacles. OPOD attributes proactive sub-agent feedback to prior decisions and converts natural-language improvement signals into process supervision through self-distillation.

Learning from execution experience can improve the orchestrator's delegation decisions. Existing approaches fall into two categories: \textbf{Learning from sparse supervision.} \aorchestra{}~\citep{ruan2026aorchestra} adopts supervised fine-tuning (SFT), while MAS-Orchestra~\citep{ke2026masorchestra} applies reinforcement learning with sparse and delayed task-level rewards. \textbf{Process-level reward shaping.} Uno-Orchestra~\citep{cui2026unoorchestra} introduces format-based rewards at each step, whereas DA-GRPO~\citep{zhang2026orchestrao1omnimodalagentorchestration} uses LLM-based rubrics to score candidate decisions. However, local rewards primarily assess output compliance or decision quality, and therefore provide limited decision attribution or directional revision signals that are grounded in the actual obstacles encountered during execution. By contrast, OPOD turns proactive sub-agent feedback into execution-grounded attribution and revision guidance for prior orchestration decisions.

\section{Methodology}

We first formulate the existing re-orchestration process (Section~\ref{sec:problem-formulation}), then introduce \ourmethod{} (Figure~\ref{fig:method1}), which interleaves sub-agent execution and harness adaptation via proactive communication (Section~\ref{sec:execution-coordination}). Finally, we use this sub-agent feedback as process supervision to optimize the orchestrator with OPOD via on-policy self-distillation (Section~\ref{sec:orchestration-optimization}).

\subsection{Problem Formulation}
\label{sec:problem-formulation}
In sub-agent-based re-orchestration, task execution alternates between global control by the orchestrator and local execution by sub-agents. Specifically, given a user goal $G$ and an orchestration horizon $T_{\mathrm{orch}}$, at round $0\le t<T_{\mathrm{orch}}$, the orchestrator either creates a sub-agent to execute a delegated task or submits the final answer:
\begin{equation}
a_t^{\mathrm{orch}}\sim\pi_\theta(\cdot\mid G,s_t^{\mathrm{orch}}),
\qquad
a_t^{\mathrm{orch}}\in\{\texttt{Delegate}(\Phi_t),\texttt{Finish}(y)\}.
\end{equation}
$\pi_\theta$ denotes the orchestration policy parameterized by $\theta$. For delegation, the orchestrator generates a harness $\Phi_t=(g_t,C_t,\Omega_t,B_t)$ for a new sub-agent, specifying the delegated subgoal, curated working context, assigned tool set, and execution budget. $y$ denotes the final answer, and the orchestration state $s_t^{\mathrm{orch}}$ summarizes the orchestrator's decision history and the local results returned by completed sub-agents.
Each $\Phi_t$ instantiates a sub-agent with action space $\mathcal A_t^{\mathrm{exe}}=\Omega_t\cup\{\texttt{Submit}\}$. Tools in $\Omega_t$ support environment interaction, whereas \texttt{Submit} ends the sub-agent's lifecycle and returns a result $o_t^{\mathrm{orch}}$ to the orchestrator. This repeats until the orchestrator terminates at round $t^*$:
\begin{equation}
\tau^{\mathrm{orch}}
=
(
s_0^{\mathrm{orch}},\texttt{Delegate}(\Phi_0),o_0^{\mathrm{orch}},
\ldots,s_{t^*}^{\mathrm{orch}},\texttt{Finish}(y),R_\tau
).
\end{equation}
Here, $R_\tau$ is the terminal reward assigned by the task evaluator.
This formulation reveals two limitations: \textbf{(1) During execution,} existing re-orchestration methods typically keep each sub-agent's harness fixed throughout its lifetime. Adapting the harness requires waiting for the current sub-agent to invoke \texttt{Submit} or time out before creating a new one. This not only delays adaptation but also causes the loss of the sub-agent's ongoing execution state. \textbf{(2) During optimization,} sub-agent feedback $o_t^{\mathrm{orch}}$ becomes available only at the end of each sub-agent lifecycle, often after dozens of execution steps, whereas the terminal reward $R_\tau$ is delayed until the entire task concludes after multiple such lifecycles. Consequently, supervision for orchestration decisions is sparse and delayed.

\begin{figure}[h]
    \centering
            \vspace{-0.7em}
    \includegraphics[width=0.985\textwidth]{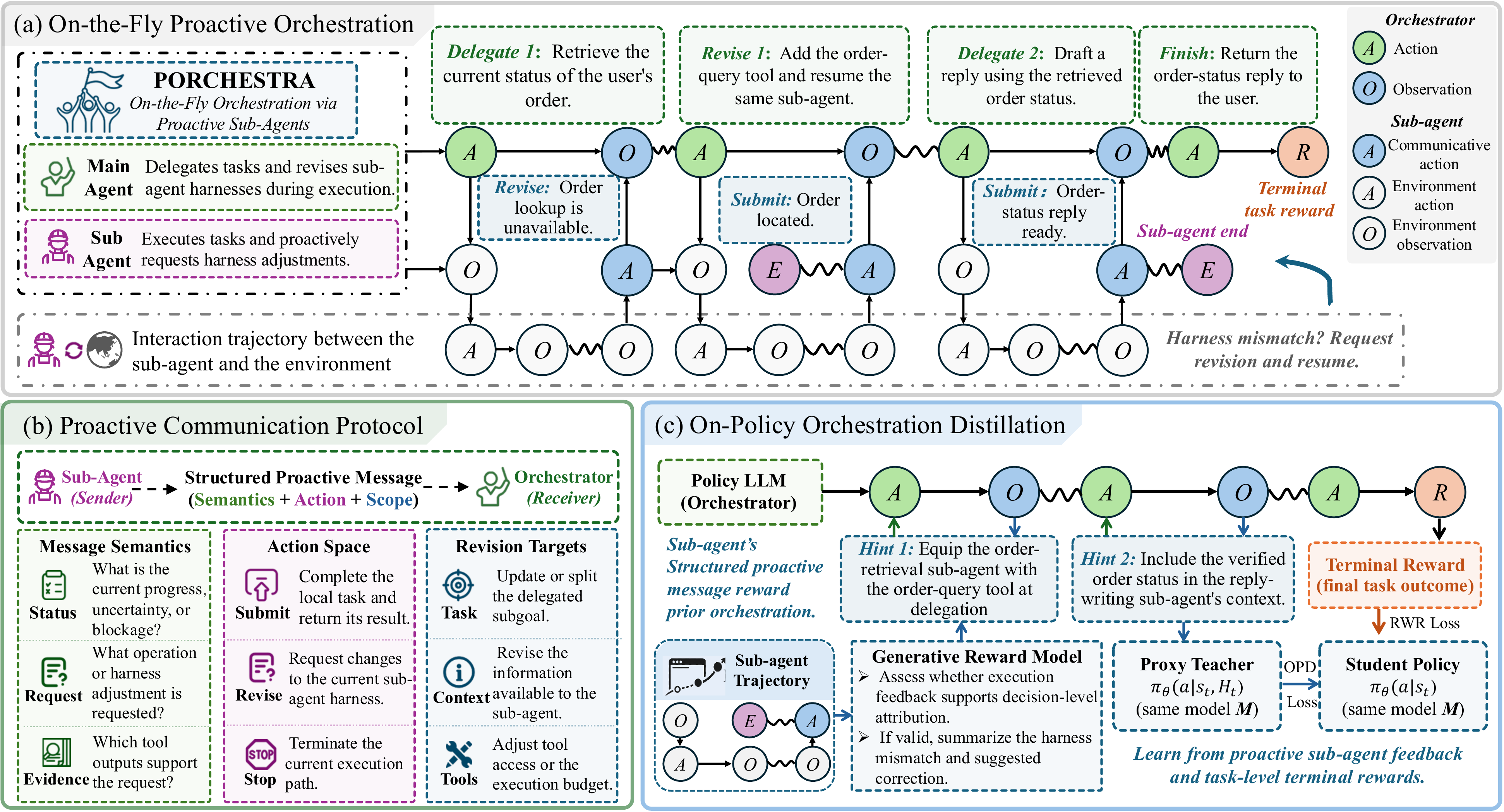}
       \vspace{-0.6em}
    \caption{Overview of \ourmethod{} and OPOD. \textbf{(a)} Sub-agent execution and harness adaptation form a coherent workflow. \textbf{(b)} Structured proactive communication guides harness revisions. \textbf{(c)} OPOD combines feedback-based self-distillation with terminal rewards to optimize the orchestrator.}
    \vspace{-0.8em}
    \label{fig:method1}
\end{figure}

\subsection{On-the-Fly Agentic Orchestration via Proactive Sub-Agents}
\label{sec:execution-coordination}

To address delayed adaptation in agentic orchestration, \ourmethod{} gives sub-agents control over initiating orchestration adjustments during execution. A sub-agent can pause at any point and request harness revisions from the orchestrator. We formalize this process through three components: \textit{(i) Sub-Agent Execution}, \textit{(ii) Proactive Communication Protocol}, and \textit{(iii) Orchestrator Coordination}.

\paragraph{Sub-Agent Execution.}
Let $M_{\mathrm{sub}}$ denote the sub-agent model. For a sub-agent instantiated by $\Phi_i$, at local step $\ell\in\{0,\ldots,B_i-1\}$, the prompt template $I_{\mathrm{sub}}$ combines the subgoal $g_i$, orchestrator-assigned context $C_i$, tool set $\Omega_i$, internal memory $h_{i,\ell}$ consisting of the accumulated local trajectory, and communication action set $\mathcal A_i^{\mathrm{comm}}$ into a sub-agent prompt. The prompt guides $M_{\mathrm{sub}}$ to consider proactive communication as a planning option when it encounters obstacles that cannot be resolved through local actions, such as task ambiguity or unavailable essential tools:
\begin{equation}
a_{i,\ell}^{\mathrm{sub}}
= M_{\mathrm{sub}}(I_{\mathrm{sub}}(g_i,C_i,\Omega_i,h_{i,\ell},\mathcal A_i^{\mathrm{comm}})),
\qquad
a_{i,\ell}^{\mathrm{sub}}
\in
\mathcal A_i^{\mathrm{sub}}
=
\Omega_i\cup\mathcal A_i^{\mathrm{comm}}.
\end{equation}

When $M_{\mathrm{sub}}$ invokes a tool from $\Omega_i$, the environment returns an observation, which is appended to the sub-agent's local trajectory. When $M_{\mathrm{sub}}$ selects a communication action, execution pauses and control passes to the orchestrator. The orchestrator can either revise the current sub-agent's harness or terminate it. In the former case, the sub-agent's harness is updated to $\Phi_i^{\ell+1}$. The update record is added to the sub-agent's internal memory as a communication observation. This mechanism allows the sub-agent to recognize harness updates alongside its prior execution history and continue task execution coherently. In the latter case, the sub-agent's lifecycle ends immediately.

\paragraph{Proactive Communication Protocol.}
Proactive sub-agent communication provides the orchestrator with an intervention point, but runtime adaptation also requires sufficient execution information. Inspired by organizational information-processing theory~\citep{galbraith1974organization}, we instantiate $\mathcal A_i^{\mathrm{comm}}$ as a standardized proactive orchestration communication protocol. Each communication action carries a structured message whose semantics are organized into three complementary signals:
\begin{itemize}[leftmargin=1.5em, itemsep=0pt, topsep=0pt]
    \item \textit{Status signal} summarizes current progress, uncertainty, and blockage from the sub-agent's perspective, helping the orchestrator understand why communication was initiated.
    \item \textit{Request signal} specifies the adjustment target and intended operation.
    \item \textit{Evidence signal} preserves key tool outputs that substantiate the request.
\end{itemize}

$\mathcal A_i^{\mathrm{comm}}$ supports three control actions: \texttt{submit} completes the local task and returns its result; \texttt{revise} requests a modification to the sub-agent configuration; and \texttt{stop} terminates the execution path. When the selected action is \texttt{revise}, the request may specify changes along two complementary dimensions. \textit{Task Scope} adjusts the subgoal $g_i$, for example by splitting or updating the task. \textit{Resource} adjusts the working context $C_i$, tool set $\Omega_i$, or execution budget $B_i$.

% We instantiate each message as schema-constrained JSON. The schema makes the requested operation, target configuration, and supporting evidence explicit and machine-checkable, separating executable updates from free-form rationales.

\paragraph{Orchestrator Coordination.}
From the orchestrator's perspective, \ourmethod{} extends re-orchestration from decisions made only after a sub-agent returns to continuous coordination throughout its lifecycle. For task $G$, the orchestrator begins with \texttt{Delegate}($\Phi_i^0$). When structured feedback arrives, the orchestrator updates the global orchestration state $s_t^{\mathrm{orch}}$ and selects $a_t^{\mathrm{orch}}\in\{\texttt{Delegate}(\Phi_j^0),\texttt{Revise}(\Phi_i^{\ell+1}),\texttt{Stop}(i),\texttt{Finish}(y)\}$. Here, superscripts distinguish configuration versions within one delegation. \ourmethod{} extends the orchestration trajectory into an iterative process in which proactive sub-agent feedback and orchestration decisions interleave:
\begin{equation}
\tau_{\mathrm{PORCH}}^{\mathrm{orch}}
=
\langle
s_0^{\mathrm{orch}},
\texttt{Delegate}(\Phi_0^0),
o_0^{\mathrm{orch}},
s_1^{\mathrm{orch}},
\texttt{Revise}(\Phi_0^1),
o_1^{\mathrm{orch}},
\ldots,
s_{t^*}^{\mathrm{orch}},
\texttt{Finish}(y),
R_\tau
\rangle.
\end{equation}

Consequently, between \texttt{Delegate}($\Phi_i$) and the sub-agent's return, \ourmethod{} uses more timely and informative proactive feedback to adjust the subgoal and harness in $\Phi_i$ as tasks progress and environmental needs change. In contrast, existing sub-agent-based re-orchestration methods adjust only after sub-agent completion, creating tension between coherent execution and timely adaptation:

%This asymmetric allocation also makes \ourmethod{} model-agnostic and plug-and-play. Because it only augments the frozen sub-agent interface with \texttt{Communicate}, it leaves sub-agent parameters and existing tool-use policies unchanged; deployment with heterogeneous sub-agents requires only the communication tool together with its protocol schema and usage instructions.

%\paragraph{Relation to Existing Re-Orchestration.}
%Existing re-orchestration methods typically adjust the organization only after a sub-agent lifecycle ends or rely on predefined conditions to determine when intervention should occur. \ourmethod{} instead embeds orchestration adjustments into sub-agent execution, allowing configuration updates and local task progress to interleave along the same trajectory. The following proposition formalizes the intervention delay induced by restricting re-orchestration to delegation boundaries.

\begin{proposition}[Re-Orchestration After Sub-Agent Completion Incurs Adaptation Delay]
\label{prop:boundary-delay}
Suppose a sub-agent requires at least $H_{\min}$ steps to maintain coherent execution, while re-orchestration is permitted only after its lifecycle ends. If a mismatch occurs uniformly over horizon $H$, then the probability of intervention within a recovery window $\delta$ satisfies $\Pr(D_H\leq\delta)\leq\min\left\{1,\frac{\delta}{H_{\min}}\right\}$.
\end{proposition}
When $H_{\min}\gg\delta$, this bound scales inversely with $H_{\min}$: the longer the horizon required for coherence, the lower the upper bound on timely intervention. \ourmethod{} allows sub-agents to trigger adjustments during execution, so intervention need not wait until the $H$-step execution ends. We summarize this interleaved inference process in Appendix~\ref{app:algorithms} (Algorithm~\ref{alg:porchestra-inference}).
% Proactive feedback also provides intermediate supervision for orchestrator optimization, as discussed next.

\subsection{On-Policy Orchestration Distillation}
\label{sec:orchestration-optimization}

% Compared with eliciting orchestration solely through prompting, optimizing the orchestrator to improve its decisions from past execution experience offers an attractive direction. However, learning an orchestration policy faces a fundamental difficulty: each orchestration decision typically receives local feedback only after the corresponding sub-agent completes its execution, while its task-level effect often cannot be evaluated until the overall task terminates, making supervision extremely sparse and expensive to obtain. The key challenge is therefore to provide more timely and informative feedback for orchestration decisions and convert it into effective supervision.
Compared with eliciting orchestration through prompting, optimizing the orchestrator from execution experience offers an attractive direction. However, supervision for orchestration decisions is sparse and expensive to obtain, making it essential to provide more timely and informative feedback and convert it into effective supervision. \ourmethod{}'s interleaved process provides a natural basis for this objective: proactive sub-agent feedback not only guides runtime adaptation but can also reveal how preceding orchestration decisions could be improved. Motivated by this, we introduce On-Policy Orchestration Distillation (OPOD), which converts sub-agent feedback into natural-language gradients, constructs a feedback-augmented self-teacher using the same model, and distills its decision distribution into the orchestration policy on trajectories generated by the current policy.

\paragraph{Natural-Language Gradient Generation.}
To turn structured feedback $o_t^{\mathrm{orch}}$ into supervision for improving prior orchestration decisions, we first assess whether the issues it reveals can be attributed to the original orchestration choices. For example, subgoal ambiguity exposed during execution may reflect unclear task delegation and guide its improvement; by contrast, if the user changes task requirements during sub-agent execution, the resulting need to revise the subgoal should not be attributed to that decision. When sufficient evidence supports attribution, we extract cues for harness improvement from proactive sub-agent feedback and distill them into a natural-language gradient, providing concrete revision guidance for prior orchestration decisions. Formally, whenever $o_t^{\mathrm{orch}}$ is returned, we use the same backbone model $M$ as the trainable orchestrator to perform attribution under a fixed generative reward-modeling instruction $I_{\mathrm{GRM}}$. Its inputs are the corresponding orchestration state and decision, the intervening sub-agent execution segment, and the returned feedback:
\begin{equation}
\{\texttt{FLAG}_t,\mathcal H_t\}
=M(I_{\mathrm{GRM}}(s_t^{\mathrm{orch}},a_t^{\mathrm{orch}},\langle h_k^{\mathrm{sub}},a_k^{\mathrm{sub}},o_k^{\mathrm{env}}\rangle_{k=1}^{K_t},o_t^{\mathrm{orch}})).
\end{equation}
Here, $\langle h_k^{\mathrm{sub}},a_k^{\mathrm{sub}},o_k^{\mathrm{env}}\rangle_{k=1}^{K_t}$ denotes the sub-agent execution segment after $a_t^{\mathrm{orch}}$ takes effect and before $o_t^{\mathrm{orch}}$ is returned, where $k$ indexes execution steps and $K_t$ is the segment length.
The flag $\texttt{FLAG}_t\in\{\texttt{true},\texttt{false}\}$ indicates whether sufficient evidence supports attribution to the orchestration decision. When $\texttt{FLAG}_t=\texttt{true}$, $\mathcal H_t$ summarizes the harness mismatch and suggested correction; otherwise, $\mathcal H_t$ is empty and the attribution unit is excluded from distillation.

\paragraph{Self-Distillation Policy Optimization.}
To use $\mathcal H_t$ as policy-level supervision, we append it to the original orchestration context and use the same model to construct a feedback-augmented teacher policy. The original orchestration policy and teacher policy are defined as
\begin{equation}
\pi_\theta(a\mid s_t^{\mathrm{orch}})
=
M(a\mid I_{\mathrm{orch}}(s_t^{\mathrm{orch}})),
\qquad
\pi_\theta^{+}(a\mid s_t^{\mathrm{orch}},\mathcal H_t)
=
M(a\mid I_{\mathrm{orch}}(s_t^{\mathrm{orch}})\oplus \mathcal H_t),
\end{equation}
where $\oplus$ denotes context concatenation. The resulting feedback-augmented policy can be viewed as an in-context policy-improvement step: $\mathcal H_t$ changes the orchestration decision distribution without updating model parameters. Because $\pi_\theta^{+}$ uses the same backbone model and differs only in its augmented context, it serves as a self-teacher without requiring an external model. For a complete trajectory $\tau$, let $\mathcal T(\tau)=\{t\mid(s_t^{\mathrm{orch}},a_t^{\mathrm{orch}})\in\tau\}$ denote its orchestration decision indices. We define the on-policy distillation (OPD) loss~\citep{agarwal2024onpolicydistillation} using full-vocabulary reverse KL:
\begin{equation}
\ell_{\mathrm{OPD}}(\tau)
=
{\textstyle\sum_{t\in\mathcal T(\tau)}}
\mathbbm{1}[\texttt{FLAG}_t=\texttt{true}]
D_{\mathrm{KL}}
(
\pi_\theta(\cdot\mid s_t^{\mathrm{orch}})
\,\|\,
\operatorname{sg}[
\pi_\theta^{+}(\cdot\mid s_t^{\mathrm{orch}},\mathcal H_t)
]
)
.
\end{equation}
where $t$ indexes orchestration decisions and $\operatorname{sg}[\cdot]$ stops gradients through the self-teacher. This loss aggregates guidance from all attribution units with valid gradients.
We retain outcome supervision via reward-weighted regression (RWR)~\citep{10.1145/1273496.1273590}, with trajectory loss:
\begin{equation}
\ell_{\mathrm{RWR}}(\tau)
=
-R_\tau
{\textstyle\sum_{t\in\mathcal T(\tau)}}
\log\pi_\theta(a_t^{\mathrm{orch}}\mid s_t^{\mathrm{orch}})
.
\end{equation}
Finally, OPOD combines process-level OPD and trajectory-level RWR supervision: $\mathcal L_{\mathrm{OPOD}}(\tau)=\ell_{\mathrm{RWR}}(\tau)+\lambda\ell_{\mathrm{OPD}}(\tau)$.
In implementation, we initialize the orchestration policy through supervised fine-tuning on expert trajectories, then optimize it with OPOD on on-policy rollouts (Algorithm~\ref{alg:opod-training}).

\begin{proposition}[OPOD Reduces Long-Horizon Attribution Error]
\label{prop:opod-stability}
Under bounded attribution noise, terminal rewards propagate through an attribution horizon $O(H_{\mathrm{orch}})$ times longer than that of OPD feedback, thereby accumulating noise over a longer execution chain. Furthermore, when the bias of OPD feedback is bounded, there exists a weight $\lambda>0$ such that OPOD yields lower expected error relative to ideal process-level supervision than RWR alone (see Appendix~\ref{app:theory} for both proofs).
\end{proposition}

\begin{table*}[t]
\centering
\small
\setlength{\tabcolsep}{2.5pt}
\renewcommand{\arraystretch}{1.11}
\setlength{\arrayrulewidth}{0.25pt}
\setlength{\belowcaptionskip}{3pt}
\caption{Performance with three backbones across three dynamic, long-horizon environments: GAIA, GAIA2, and SWE-Bench. Baselines include single-agent systems (ReAct, Mini-SWE), offline orchestration (Skill-MAS, AutoAgents), and re-orchestration (MAS$^2$, Orchestra-o1, \aorchestra{}). Best results are \textbf{bold}, and second-best results are \underline{underlined}.}
\label{tab:deep_search}
\begin{adjustbox}{width=0.99\textwidth}
\begin{tabular}{lcccccccccc}
\toprule
\multirow{2}{*}{\textbf{Method}}
& \multicolumn{4}{c}{\textbf{General AI Assistant}}
& \multicolumn{5}{c}{\textbf{General AI Assistant 2}}
& \multirow{2}{*}{\textbf{SWE-Bench}} \\
\cmidrule(lr){2-5} \cmidrule(lr){6-10}
& Level 1 & Level 2 & Level 3 & Average
& Execution & Search & Adaptability & Time & Average & \\
\midrule

\multicolumn{11}{l}{\textit{\textbf{DeepSeek-V3.2}}} \\
ReAct
& \underline{79.2} & 58.1 & \underline{50.0} & \underline{63.6} & \underline{34.4} & 56.3 & 21.9 & \underline{21.9} & 33.6 & 62.0 \\
Mini-SWE
& 64.2 & \underline{64.0} & 26.9 & 58.2 & 25.0 & 50.0 & 18.8 & 12.5 & 26.6 & \underline{70.0} \\

\hdashline[4pt/2pt]
% SWE-Bench merged final, same 100 tasks: SubAgent DeepSeek-V3.2; original 60/100; rerun 0/2, 1 additional task stopped; merged 60/100.
Skill-MAS
& 67.9 & 58.1 & 30.8 & 57.0 & 21.9 & 43.8 & 0.0 & 0.0 & 16.4 & 60.0 \\
% SWE-Bench final, same 100 tasks: Main Gemini-low / SubAgent DeepSeek-V3.2; 30/100; 43 timeouts, 1 planning error.
AutoAgents
& 64.2 & 40.7 & 34.6 & 47.3 & 15.6 & 59.4 & 0.0 & 6.3 & 20.3 & 30.0 \\

\hdashline[4pt/2pt]
% SWE-Bench: merged API reruns, 62/100 (original run: 39/100).
MAS$^2$
& 47.2 & 39.5 & 19.2 & 38.8 & 25.0 & 56.3 & 12.5 & 3.1 & 24.2 & 62.0 \\
Orchestra-o1
& \textbf{81.1} & 61.6 & 30.8 & 63.0 & \textbf{43.8} & 59.4 & \underline{25.0} & 9.4 & 34.4 & 66.0 \\
\aorchestra{}
& 73.6 & 55.8 & 46.2 & 60.0 & \textbf{43.8} & \underline{75.0} & \underline{25.0} & 15.6 & \underline{39.8} & 65.0 \\

\hdashline[4pt/2pt]
\rowcolor{arporow}
\ourmethod{}
& \underline{79.2} & \textbf{72.1} & \textbf{53.8} & \textbf{71.5} & \textbf{43.8} & \textbf{78.1} & \textbf{43.8} & \textbf{25.0} & \textbf{47.7} & \textbf{71.0} \\

\midrule
\multicolumn{11}{l}{\textit{\textbf{Gemini-3-Flash}}} \\
% SWE-Bench official evaluation: Gemini sub-agent, 53/100 passed.
ReAct
& \underline{84.9} & 66.3 & 46.2 & 69.1 & \underline{50.0} & \underline{84.4} & \underline{28.1} & \underline{25.0} & \underline{46.9} & 53.0 \\
% GAIA final: Gemini Mini-SWE; L1 43/53, L2 66/86, L3 15/26; total 124/165.
Mini-SWE
& 81.1 & \textbf{76.7} & 57.7 & 75.2 & 43.8 & 25.0 & 25.0 & \underline{25.0} & 29.7 & 75.0 \\

\hdashline[4pt/2pt]
% SWE-Bench merged final, same 100 tasks: SubAgent Gemini-low; original 49/100; rerun 19/42, 1 additional task stopped; merged 68/100.
% GAIA final: Gemini-low; L1 43/53, L2 53/86, L3 12/26; total 108/165.
Skill-MAS
& 81.1 & 61.6 & 46.2 & 65.5 & 28.1 & 62.5 & 9.4 & 6.3 & 26.6 & 68.0 \\
% SWE-Bench final, same 100 tasks: Main Gemini-low / SubAgent Gemini-low; 67/100; 1 timeout, 1 planning error.
% GAIA final: Gemini-low; L1 37/53, L2 44/86, L3 14/26; total 95/165; elapsed approximately 2 hours.
AutoAgents
& 69.8 & 51.2 & 53.8 & 57.6 & 25.0 & 46.9 & 0.0 & 3.1 & 18.8 & 67.0 \\

\hdashline[4pt/2pt]
% GAIA final: Gemini Orchestra-o1; L1 46/53 (86.79%), L2 58/86 (67.44%), L3 15/26 (57.69%); total 119/165 (72.12%).
Orchestra-o1
& \textbf{86.8} & 67.4 & 57.7 & 72.1 & \underline{50.0} & \underline{84.4} & 25.0 & 15.6 & 43.8 & 69.0 \\
% SWE-Bench final: Main Gemini-low / SubAgent Gemini-low; 59/100.
% GAIA final: Gemini-low; L1 41/53, L2 49/86, L3 11/26; total 101/165; elapsed approximately 3 hours 26 minutes.
MAS$^2$
& 77.4 & 57.0 & 42.3 & 61.2 & \textbf{56.3} & 75.0 & 21.9 & 0.0 & 38.6 & 59.0 \\
% SWE-Bench final: Gemini \aorchestra{}; 76/100.
% GAIA final: Main Gemini / SubAgent Gemini-low; L1 46/53 (86.79%), L2 64/86 (74.42%), L3 16/26 (61.54%); total 126/165 (76.36%).
\aorchestra{}
& \textbf{86.8} & \underline{74.4} & \underline{61.5} & \underline{76.4} & 46.9 & 75.0 & 18.8 & 18.8 & 39.8 & \underline{76.0} \\

\hdashline[4pt/2pt]
\rowcolor{arporow}
\ourmethod{}
& \textbf{86.8} & \textbf{76.7} & \textbf{65.4} & \textbf{78.2} & \textbf{56.3} & \textbf{90.6} & \textbf{31.3} & \textbf{37.5} & \textbf{53.9} & \textbf{77.0} \\

\midrule
\multicolumn{11}{l}{\textit{\textbf{DeepSeek-V4-Flash}}} \\
ReAct
& 79.2 & \underline{72.1} & \textbf{61.5} & 72.7 & 37.5 & 75.0 & \underline{37.5} & \textbf{28.1} & \underline{44.5} & 71.0 \\
Mini-SWE
& 76.9 & 67.4 & \underline{53.8} & 68.3 & 40.6 & 78.1 & 25.0 & 18.8 & 40.6 & 73.0 \\

\hdashline[4pt/2pt]
% SWE-Bench final: Skill-MAS / DeepSeek-V4-Flash; 76/100.
Skill-MAS
& 73.6 & 59.3 & 46.2 & 61.8 & 28.1 & 65.6 & 0.0 & 3.1 & 24.2 & 76.0 \\
% SWE-Bench final, same 100 tasks: Main Gemini-low / SubAgent DeepSeek-V4; 69/100; 4 connection errors, 4 timeouts, 1 planning error.
AutoAgents
& \underline{86.8} & 62.8 & 46.2 & 67.9 & 21.9 & 53.1 & 0.0 & 0.0 & 18.8 & 69.0 \\

\hdashline[4pt/2pt]
Orchestra-o1
& 83.0 & 57.0 & 38.5 & 62.4 & 50.0 & 77.4 & 25.0 & \underline{25.0} & 43.8 & 75.0 \\
% SWE-Bench: merged API reruns, 78/100 (original run: 56/100).
MAS$^2$
& 69.8 & 61.6 & 46.2 & 61.8 & 31.3 & 75.0 & 3.1 & 0.0 & 26.6 & \underline{78.0} \\
\aorchestra{}
& \underline{86.8} & \textbf{75.6} & 50.0 & \underline{75.2} & \underline{53.1} & \underline{81.3} & 24.1 & 15.6 & 43.0 & 71.0 \\

\hdashline[4pt/2pt]
\rowcolor{arporow}
\ourmethod{}
& \textbf{88.7} & \textbf{75.6} & \textbf{61.5} & \textbf{77.6} & \textbf{59.4} & \textbf{87.5} & \textbf{46.9} & \textbf{28.1} & \textbf{55.5} & \textbf{79.0} \\

\bottomrule
\end{tabular}
\end{adjustbox}
\vspace{-1.3em}
\end{table*}

\section{Experiments}

We evaluate our framework and optimization method through four research questions:
\begin{itemize}[leftmargin=1.5em, itemsep=0pt, topsep=0pt]
    \item \textbf{RQ1:} Does proactive orchestration improve task success across benchmarks and backbone models? We compare \ourmethod{} with single-agent and orchestration baselines (Section~\ref{sec:exp-execution}).
    \item \textbf{RQ2:} Does OPOD improve orchestrator optimization? We compare training methods and examine the contributions of process feedback and outcome supervision (Section~\ref{sec:exp-optimization}).
    \item \textbf{RQ3:} How do \ourmethod{} and OPOD affect cost--performance trade-offs? We analyze inference costs and the efficiency of trained small orchestrators (Section~\ref{sec:exp-cost}).
    \item \textbf{RQ4:} When and why does proactive orchestration help? We examine task length and dynamicity, and compare proactive feedback with predefined triggers and periodic polling (Section~\ref{sec:exp-framework}).
\end{itemize}

\subsection{Experimental Setup}

\paragraph{Benchmarks.}
We evaluate \ourmethod{} on three long-horizon agent benchmarks with distinct interaction settings. \textbf{GAIA}~\citep{mialon2024gaia} evaluates the ability of deep-research assistants to complete tasks by using real tools and accessing real web pages; its interactions are affected by environmental changes and tool stability. \textbf{GAIA2}~\citep{froger2026gaia} is a challenging read--write interaction benchmark that emphasizes agents' adaptability in dynamic and asynchronous environments, such as adjusting task goals and subsequent actions based on replies received after sending an email. \textbf{SWE-Bench Verified}~\citep{jimenez2024swebench} evaluates software-engineering agents by requiring them to resolve real-world GitHub issues. We use task success rate as the evaluation metric.

\paragraph{Model \& Baselines.}
For inference-time evaluation, we compare \ourmethod{} with three categories of baselines: \textbf{(1) single-agent baselines,} including ReAct~\citep{yao2023react} and Mini-SWE~\citep{yang2024sweagent}; \textbf{(2) offline orchestration systems,} including Skill-MAS~\citep{lin2026skill} and AutoAgents~\citep{ijcai2024p3}; and \textbf{(3) re-orchestration systems,} including MAS$^2$~\citep{wang2026mas}, \aorchestra{}~\citep{ruan2026aorchestra}, and Orchestra-o1~\citep{zhang2026orchestrao1omnimodalagentorchestration}. Following the \aorchestra{} setup, orchestrated systems use \modelname{Gemini-3-Flash}~\citep{googledeepmind2025gemini3flash} as the orchestrator. All systems are evaluated with \modelname{Gemini-3-Flash}, \modelname{DeepSeek-V3.2}~\citep{deepseekai2025deepseekv32pushingfrontieropen}, and \modelname{DeepSeek-V4-Flash}~\citep{deepseekai2026deepseekv4highlyefficientmilliontoken} as execution backbones. For training-time evaluation, we train \modelname{Qwen3.5-9B}~\citep{qwen3.5} orchestrators across Orchestra-o1, \aorchestra{}, and \ourmethod{}. Methods include SFT, RWR, DA-GRPO~\citep{zhang2026orchestrao1omnimodalagentorchestration}, and OPOD; each is applied to frameworks compatible with its feedback and execution interface.
See Appendices~\ref{app:implementation} and~\ref{app:prompts-protocol} for implementation details and prompts.

\subsection{Effectiveness of Proactive Orchestration in Execution (RQ1)}
\label{sec:exp-execution}

\noindent\textbf{Obs. \ding{182}: Existing orchestration methods struggle to adapt reliably to dynamic environments.}
Table~\ref{tab:deep_search} shows that offline orchestration (Skill-MAS and AutoAgents) substantially trails single-agent baselines in average GAIA2 success across all three backbones. For Adaptability tasks involving changing task requirements, AutoAgents scores 0\% with all three backbones. Current re-orchestration baselines improve some results but remain mixed against single-agent baselines: with \modelname{Gemini-3-Flash}, \aorchestra{} reaches 76.4\% on GAIA, above ReAct's 69.1\%, but 39.8\% on GAIA2, below ReAct's 46.9\%. These results suggest static harnesses or delayed harness updates represent important bottlenecks to consistent gains from orchestration in dynamic environments.

\begin{table*}[t]
\centering
\small
\setlength{\tabcolsep}{2.5pt}
\renewcommand{\arraystretch}{1.01}
\setlength{\belowcaptionskip}{3pt}
\caption{Results and ablation of OPOD across orchestration harnesses and algorithms.}
\label{tab:optimization}
\begin{adjustbox}{max width=\textwidth}
\begin{tabular}{llccccccccc}
\toprule
\multirow{2}{*}{\textbf{Orchestration}}
& \multirow{2}{*}{\textbf{Optimization}}
& \multicolumn{4}{c}{\textbf{General AI Assistant}}
& \multicolumn{5}{c}{\textbf{General AI Assistant 2}} \\
\cmidrule(lr){3-6} \cmidrule(lr){7-11}
& & Level 1 & Level 2 & Level 3 & Average
& Execution & Search & Adaptability & Time & Average \\
\midrule

\multirow{4}{*}{\textbf{Orchestra-o1}}
% GAIA2 strict: Execution 5/32, Search 12/32, Adaptability 0/32, Time 0/32; total 17/128.
% Missing persisted results count as incorrect: Adaptability 6, Search 1; total 7.
& Qwen3.5-9B & 69.8 & 57.0 & 50.0 & 60.0 & 15.6 & 37.5 & 0.0 & 0.0 & 13.3 \\
% GAIA2 SFT final: Execution 8/32, Search 22/32, Adaptability 3/32, Time 4/32; total 37/128.
% GAIA O1-SFT Legacy, 4-hour limit: L1 43/53, L2 52/86, L3 11/26; total 106/165.
& SFT & \textbf{81.1} & 60.5 & 42.3 & 64.2 & 25.0 & 68.8 & 9.4 & 12.5 & 28.9 \\
% GAIA2 DA-GRPO, mean of two runs: Execution 35.94%, Search 68.75%, Adaptability 6.25%, Time 12.50%; average 30.86%.
% GAIA DA-GRPO: L1 41/53, L2 58/86, L3 13/26; total 112/165 (67.88%).
& DA-GRPO & 77.4 & 67.4 & 50.0 & 67.9 & 35.9 & 68.8 & 6.3 & 12.5 & 30.9 \\
% GAIA1 + DPv4 RWR A, r32, 1 epoch; adapter saved: L1 42/53 (79.25%), L2 54/86 (62.79%), L3 11/26 (42.31%); total 107/165 (64.85%).
% GAIA2 RWR ep3: Execution 34.38%, Search 75.00%, Adaptability 9.38%, Time 12.50%; average 32.81%.
& SFT + RWR & 79.2 & 62.8 & 42.3 & 64.8 & 34.4 & 75.0 & 9.4 & 12.5 & 32.8 \\

\midrule
\multirow{3}{*}{\textbf{\aorchestra{}}}
& Qwen3.5-9B & 70.4 & 63.6 & 46.2 & 63.0 & 31.3 & 75.0 & 3.1 & 9.4 & 29.7 \\
& SFT & 76.7 & 69.8 & 57.7 & 70.1 & 35.9 & 75.0 & 26.0 & 18.8 & 39.1 \\
& SFT + RWR & 75.5 & 74.4 & 61.5 & 72.7 & 43.8 & 80.6 & 19.4 & 16.1 & 40.0 \\

\midrule
& Qwen3.5-9B & 73.0 & 66.3 & 60.3 & 67.5 & 31.3 & 65.6 & 12.9 & 10.0 & 29.7 \\
& SFT & \underline{80.2} & 70.9 & 61.5 & 72.4 & \underline{45.3} & 75.0 & 34.4 & 20.6 & 43.8 \\
\rowcolor{arporow}\cellcolor{white}
% Author update: GAIA2 OPOD Execution = 14/32 (43.75%), Adaptability = 37.5%; retaining Search and Time gives 49.2% overall.
& \textbf{OPOD} & \textbf{81.1} & \textbf{77.3} & \textbf{73.1} & \textbf{77.9} & 43.8 & \textbf{84.4} & \textbf{37.5} & \textbf{31.3} & \textbf{49.2} \\
\rowcolor{arporow}\cellcolor{white}
& \textit{w/o RWR} & \textbf{81.1} & \underline{75.6} & 65.4 & 75.8 & 40.6 & \underline{79.7} & \underline{35.9} & \underline{28.1} & \underline{46.1} \\
\rowcolor{arporow}\cellcolor{white}\multirow{-5}{*}{\textbf{\ourmethod{}}}
& \textit{w/o OPD} & \underline{80.2} & \underline{75.6} & \underline{69.2} & \underline{76.1} & \textbf{46.9} & \textbf{84.4} & 29.7 & 23.4 & \underline{46.1} \\

\bottomrule
\end{tabular}
\end{adjustbox}
\vspace{-1.3em}
\end{table*}

\noindent\textbf{Obs. \ding{183}: Proactive orchestration outperforms existing methods.}
With all three sub-agent backbones, \ourmethod{} achieves the highest average success rates on GAIA and GAIA2, as well as the highest success rate on SWE-Bench. In particular, on the more dynamic GAIA2 benchmark, it reaches 55.5\% with \modelname{DeepSeek-V4-Flash}, a 24.6\% relative improvement over the strongest baseline, ReAct, at 44.5\%. Notably, \ourmethod{}'s sub-agents use a ReAct-style execution mechanism. This shows that \ourmethod{} seamlessly integrates immediate harness adaptation into sub-agent execution, alleviates the adaptation delay of existing orchestration methods, and thereby more fully realizes the gains from orchestration on long-horizon tasks.

\subsection{Effectiveness of OPOD in Optimization (RQ2)}
\label{sec:exp-optimization}

\noindent\textbf{Obs. \ding{184}: OPOD effectively optimizes small orchestrators for \ourmethod{}.}
Table~\ref{tab:optimization} compares \modelname{Qwen3.5-9B} across orchestration frameworks and optimization methods. Overall, training improves small-orchestrator success on both benchmarks. Across frameworks, \ourmethod{} achieves the highest or joint-highest average success before training and retains this lead under the same training methods, preserving its framework advantage. With OPOD, \ourmethod{} further reaches 77.9\% on GAIA and 49.2\% on GAIA2, improving over its SFT initialization by 5.5 and 5.4 percentage points, respectively, and achieving the highest average success reported in the table.

\noindent\textbf{Obs. \ding{185}: Process feedback and outcome supervision play complementary roles in orchestrator optimization.}
Removing either OPD or RWR reduces average success on both benchmarks. In particular, without OPD, success drops by 7.8 and 7.9 percentage points on GAIA2's Adaptability and Time tasks. These results support the synergy between \ourmethod{}'s interleaved orchestration and OPOD on dynamic tasks: the former enables on-demand harness revisions during sub-agent execution; the latter converts feedback from these interactions into process-level supervision to improve orchestration. Removing RWR also lowers average success on GAIA and GAIA2 by 2.1 and 3.1 points, respectively, supporting the complementary value of task-level outcome supervision.

\begin{figure}[!t]
    \centering
    \scalebox{0.98}{%
    \begin{minipage}{\textwidth}
    \centering
    \includegraphics[width=\textwidth]{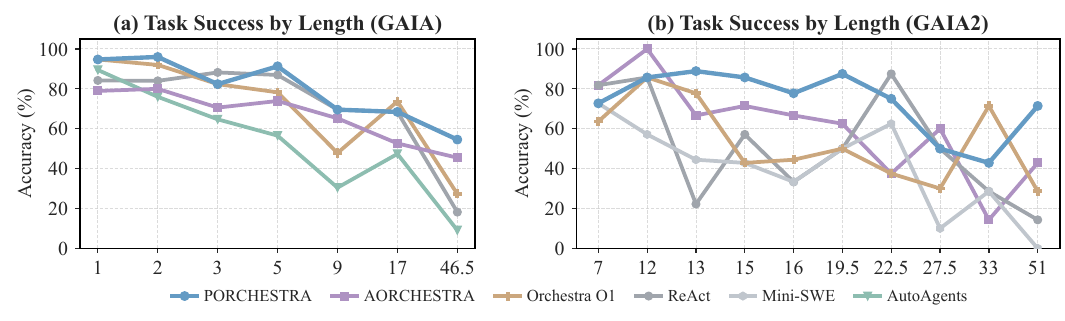}

    \vspace{-10pt}
    \input{pic/analysis/gaia2_dynamicity_layout.tex}
    \newsavebox{\dynamicitybox}
    \sbox{\dynamicitybox}{\dynamicitygraphic}
    \noindent\begin{minipage}[t]{0.448\textwidth}
        \vspace{0pt}
        \centering
        \fontfamily{ptm}\selectfont
        {\fontsize{8.16429284}{10}\selectfont\noindent\makebox[\linewidth][r]{\parbox{1.05\wd\dynamicitybox}{\centering\hspace*{18pt}\textbf{(c) Task Success Dynamic Level}\par}}\par}
        \vspace{0pt}
        \noindent\makebox[\linewidth][r]{\resizebox{1.05\wd\dynamicitybox}{\ht\dynamicitybox}{\usebox{\dynamicitybox}}}
    \end{minipage}\hfill
    \begin{minipage}[t]{0.532\textwidth}
        \vspace{0pt}
        \centering
        \fontfamily{ptm}\selectfont
        {\fontsize{8.16429284}{10}\selectfont\textbf{(d) Comparison of Feedback Mechanisms}\par}
        \vspace{0pt}
        \begingroup
        \fontsize{7}{8.4}\selectfont
        \setlength{\tabcolsep}{0.006\linewidth}
        \setlength{\arrayrulewidth}{0.25pt}
        \setlength{\aboverulesep}{1pt}
        \setlength{\belowrulesep}{1pt}
        \renewcommand{\arraystretch}{1.08}% Shorter rows; preserve the original font size.
        \setlength{\extrarowheight}{0pt}
        \newsavebox{\feedbackbox}
        \def\feedbacktable{%
        \begin{tabular}{@{}lccccc@{}}
        \toprule
        & \multicolumn{3}{c}{\textbf{Controlled intervention}}
        & \multicolumn{2}{c}{\textbf{End-to-end evaluation}} \\
        \cmidrule(lr){2-4}\cmidrule(l){5-6}
        \textbf{Mechanism}
        & Recovery $\uparrow$
        & Delay $\downarrow$
        & Cost $\downarrow$
        & Adaptability $\uparrow$
        & Overall $\uparrow$ \\
        \midrule
        % Trigger end task-level, two-run pooled average: Execution 27/64, Search 51/64, Adaptability 14/64 (21.88%), Time 17/64; overall 109/256 (42.58%).
        Trigger (end) & N/A & N/A & N/A & 21.9 & 42.6 \\
        % Trigger error, one GAIA2 mini128 run: Execution 16/32 (50.00%), Search 29/32 (90.625%), Adaptability 5/32 (15.625%), Time 5/32 (15.625%); overall 55/128 (42.96875%).
        Trigger (error) & N/A & N/A & N/A & 15.6 & 43.0 \\
        \hdashline[4pt/2pt]
        % Polling-5 task-level success: 45.31%; displayed as a ratio to two decimals.
        Polling (5) & 86.0 & 2.11 & 0.23 & 31.3 & 45.3 \\
        Polling (10) & 70.0 & 5.81 & 0.11 & 21.9 & 43.8 \\
        Polling (20) & 28.0 & 12.86 & 0.03 & 18.8 & 46.9 \\
        \hdashline[4pt/2pt]
        \rowcolor{arporow}
        \textbf{Proactive} & \textbf{92.0} & 1.67 & 0.09 & 46.9 & 55.5 \\
        \bottomrule
        \end{tabular}%
        }
        \sbox{\feedbackbox}{\feedbacktable}
        \noindent\resizebox{\linewidth}{!}{\usebox{\feedbackbox}}

        \endgroup
    \end{minipage}
    \end{minipage}%
    }\par
    \caption{Analysis of \ourmethod{}. Task success by (a,b) task length and (c) dynamicity. (d) Intervention mechanisms under controlled intervention and end-to-end evaluation. Recovery is in percent, delay in execution steps (including idle), and cost in CA per case.}
    \label{fig:length-accuracy}
    \label{fig:dynamicity-accuracy}
\vspace{-1.3em}
\end{figure}

\subsection{Cost Analysis (RQ3)}
\label{sec:exp-cost}
\noindent
\begin{wrapfigure}{r}{0.62\textwidth}
    \vspace{-\intextsep}
    \centering
    \includegraphics[width=\linewidth]{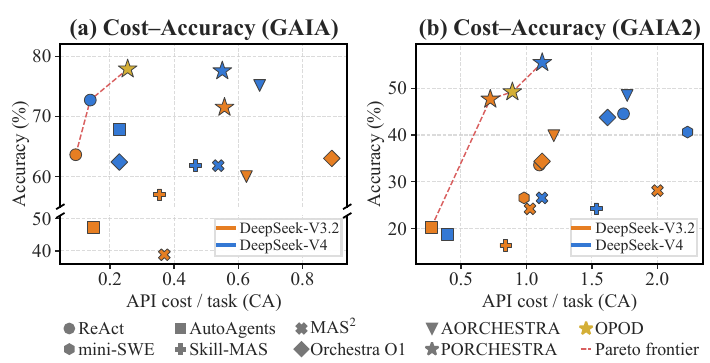}
    \par\vspace{0pt}
    \setlength{\abovecaptionskip}{0pt}
\caption{Cost--performance trade-offs on GAIA and GAIA2.}
    \label{fig:cost-pareto}
\end{wrapfigure}
\noindent\textbf{Obs. \ding{186}: \ourmethod{} improves the trade-off between task performance and inference cost.}
Figure~\ref{fig:cost-pareto} shows that \ourmethod{} achieves the highest success for each backbone on GAIA and GAIA2 at competitive cost. On GAIA2 with \modelname{DeepSeek-V4-Flash}, it exceeds ReAct by 11.0 percentage points at 36\% lower per-task API cost and Orchestra-o1 by 11.7 points at 31\% lower cost. Its efficiency stems from (1) independent sub-agent contexts and task-relevant tool sets, which substantially reduce context overhead amid GAIA2's large tool set and extensive application feedback; and (2) on-demand configuration changes during execution, which avoid wasted inference under mismatched harnesses.

\noindent\textbf{Obs. \ding{187}: OPOD extends the cost--performance frontier.}
Figure~\ref{fig:cost-pareto} shows the OPOD-trained \modelname{Qwen3.5-9B} orchestrator on the Pareto frontier of both GAIA and GAIA2. By reducing reliance on larger orchestrators, OPOD extends the frontier toward comparable performance at lower cost.

\subsection{Tracing the Gains from Proactive Orchestration (RQ4)}
\label{sec:exp-framework}

\noindent\textbf{Obs. \ding{188}: \ourmethod{} shows a clear advantage on longer and more dynamic tasks.}
(1) \textbf{Task length.} For each task solved by at least one method, we estimate its required length by the fewest environment actions among successful trajectories and group tasks accordingly. As shown in Figure~\ref{fig:length-accuracy}(a,b), \ourmethod{}'s advantage becomes more pronounced in longer task groups on both GAIA and GAIA2.
(2) \textbf{Task dynamicity.} Figure~\ref{fig:dynamicity-accuracy}(c) groups GAIA2 tasks by the number of future exogenous events. Notably, \ourmethod{} has its largest relative advantage in the group with two future events, representing moderate-to-high dynamicity. Together, these findings highlight the distinctive value of proactive orchestration for long-horizon, dynamic interaction tasks.

\noindent\textbf{Obs. \ding{189}: End-to-end performance gains trace to seamless adaptation enabled by proactive feedback.}
To trace \ourmethod{}'s performance gains, we replace proactive feedback with two triggering schemes: \textit{Trigger (end/error)} invokes the orchestrator when a sub-agent finishes or encounters an error; \textit{Polling (5/10/20)} re-orchestrates every 5, 10, or 20 execution steps. Figure~\ref{fig:dynamicity-accuracy}(d) compares them under controlled intervention and end-to-end evaluation. (1) \textbf{Controlled intervention.} We randomly revoke tool access during sub-agent execution to simulate tool unavailability. More frequent polling improves recovery and reduces response delay at higher cost. Proactive feedback achieves 92\% recovery with a mean delay of 1.67 steps. (2) \textbf{End-to-end evaluation.} On GAIA2, more frequent polling improves performance on the Adaptability subset but can disrupt execution continuity and reduce overall success; predefined triggers struggle to detect adjustment needs without explicit errors. By enabling timely adaptation with execution continuity, \ourmethod{} outperforms these variants overall, translating seamless intervention into higher task success.

\section{Conclusion}
We introduce \ourmethod{}, a proactive orchestration framework where sub-agents proactively communicate with the orchestrator at any logical point, enabling seamless harness adaptation in dynamic, long-horizon tasks. OPOD, in turn, converts this feedback into process-level supervision to optimize the orchestrator. Experiments across three benchmarks show improved task performance, favorable cost--performance trade-offs, and more effective training of small orchestrators.

\subsection*{AI use statement}
We used generative AI tools to assist with literature research, coding, translation and language editing, interpretation of experimental results, drafting figure captions and labels, and manuscript formatting. All AI-assisted content was reviewed by the authors, who take responsibility for the final paper, including its claims, analyses, and artifacts.

\subsection*{Reproducibility statement}
We provide the inference and training algorithms in Appendix~\ref{app:algorithms}, benchmark descriptions and implementation settings in Appendix~\ref{app:implementation}, theoretical proofs in Appendix~\ref{app:theory}, and prompt templates in Appendix~\ref{app:prompts-protocol}. An anonymous code repository is linked in the abstract.

\bibliography{iclr2027_conference}
\bibliographystyle{iclr2027_conference}

\appendix
\clearpage
\addtocontents{toc}{\protect\setcounter{tocdepth}{2}}
\begingroup
\renewcommand{\contentsname}{Appendix Contents}
\tableofcontents
\endgroup
\clearpage
\section{Algorithms}
\label{app:algorithms}

\begingroup
\algrenewcommand{\algorithmiccomment}[1]{\hfill\textcolor{gray}{$\triangleright$ #1}}
\algtext*{EndIf}
\algtext*{EndFor}
\algtext*{EndWhile}
\subsection{\ourmethod{} Inference}
\begin{algorithm}[H]
\caption{Interleaved orchestrator and sub-agent execution in \ourmethod{}.}
\label{alg:porchestra-inference}
\begin{algorithmic}[1]
\Require orchestrator $\pi^{\mathrm{orch}}$; sub-agent model $M_{\mathrm{sub}}$; execution budgets
\Statex \textbf{Input:} User goal $G$
\Statex \textbf{Output:} Final answer $y$, or accumulated results at the budget limit
\Statex
\Statex \textbf{Orchestrator: delegation and feedback-driven coordination}
\State $s^{\mathrm{orch}}\gets\operatorname{Initialize}(G)$
\State select $\texttt{Delegate}(\Phi_i)$; create sub-agent $i$ with $h_i\gets\varnothing$
\State mark sub-agent $i$ as ready to execute
\While{budgets permit}
    \If{a sub-agent $i$ is ready to execute}
        \State \textcolor{blue}{$o^{\mathrm{orch}}\gets\operatorname{SubAgent}(i)$} \Comment{Run until feedback or completion}
        \State update $s^{\mathrm{orch}}$ with $o^{\mathrm{orch}}$, worker status, and remaining budgets
    \EndIf
    \State $a^{\mathrm{orch}}\sim\pi^{\mathrm{orch}}(\cdot\mid s^{\mathrm{orch}})$ over valid actions
    \If{$a^{\mathrm{orch}}=\texttt{Revise}(\Phi_i')$ and sub-agent $i$ is paused}
        \State \textcolor{blue}{$\Phi_i\gets\Phi_i'$; apply the task and resource changes}
        \State append the applied changes as observation $o^{\mathrm{comm}}$ to $h_i$
        \State mark the same sub-agent $i$ as ready to resume \Comment{Preserve its local history}
    \ElsIf{$a^{\mathrm{orch}}=\texttt{Delegate}(\Phi_j)$}
        \State create sub-agent $j$ with $h_j\gets\varnothing$; set $i\gets j$ and mark it ready
    \ElsIf{$a^{\mathrm{orch}}=\texttt{Stop}(i)$}
        \State terminate sub-agent $i$ and update $s^{\mathrm{orch}}$ \Comment{No worker is ready}
        \State \textbf{continue} \Comment{Orchestrator decides again}
    \ElsIf{$a^{\mathrm{orch}}=\texttt{Finish}(y)$}
        \State \Return $y$ \Comment{Complete the overall task}
    \EndIf
\EndWhile
\State \Return accumulated results in $s^{\mathrm{orch}}$
\Statex
\Statex \textbf{Sub-Agent$(i)$: local execution and proactive feedback}
\While{sub-agent $i$ has local and total execution budget}
    \State $a^{\mathrm{sub}}\gets M_{\mathrm{sub}}(I_{\mathrm{sub}}(\Phi_i,h_i,\mathcal A_i^{\mathrm{comm}}))$
    \If{$a^{\mathrm{sub}}\in\Omega_i$}
        \State $o^{\mathrm{env}}\gets\operatorname{Execute}(a^{\mathrm{sub}})$; append $(a^{\mathrm{sub}},o^{\mathrm{env}})$ to $h_i$
    \Else \Comment{\texttt{revise}, \texttt{submit}, or \texttt{stop}}
        \State \textcolor{blue}{$o^{\mathrm{orch}}\gets(\text{status},\text{request},\text{evidence})$ from $a^{\mathrm{sub}}$ and $h_i$}
        \State append $o^{\mathrm{orch}}$ to $h_i$
        \If{$a^{\mathrm{sub}}=\texttt{revise}$}
            \State pause sub-agent $i$, preserving $(\Phi_i,h_i)$ \Comment{Await the orchestrator's adjustment}
        \Else
            \State close sub-agent $i$ as completed (\texttt{submit}) or terminated (\texttt{stop})
        \EndIf
        \State \textcolor{blue}{\Return $o^{\mathrm{orch}}$ to the orchestrator} \Comment{Sub-agent initiates the handoff}
    \EndIf
\EndWhile
\State pause sub-agent $i$; \Return budget-exhaustion feedback to the orchestrator
\end{algorithmic}
\end{algorithm}
Algorithm~\ref{alg:porchestra-inference} describes inference from the complementary perspectives of the orchestrator and the sub-agent. The orchestrator delegates tasks and responds to feedback, while the sub-agent executes local actions and proactively returns control when a harness adjustment is needed. A revision updates a paused worker's configuration and records the applied changes in its local history, allowing the same worker to resume with its accumulated context.

\clearpage
\subsection{OPOD Training}
\begin{algorithm}[H]
\caption{Feedback-guided policy optimization with OPOD.}
\label{alg:opod-training}
\begin{algorithmic}[1]
\Require expert data $\mathcal D_{\mathrm{SFT}}$; task distribution $\mathcal D$; OPD weight $\lambda$
\Statex \textbf{Input:} Base model $M_{\theta_0}$
\Statex \textbf{Output:} Optimized orchestration policy $\pi_\theta$
\State $\theta_{\mathrm{SFT}}\gets\operatorname{SFT}(\theta_0,\mathcal D_{\mathrm{SFT}})$; $\theta\gets\theta_{\mathrm{SFT}}$ \Comment{Initialize the student}
\State freeze $M_{\theta_0}$ and $M_{\theta_{\mathrm{SFT}}}$ \Comment{Base GRM and SFT teacher}
\State $\mathcal B\gets\operatorname{Collect}(\pi_{\theta_{\mathrm{SFT}}},\mathcal D)$ \Comment{Inference process (Algorithm~\ref{alg:porchestra-inference})}
\State $\mathcal D_R\gets\operatorname{SuccessfulDecisions}(\mathcal B)$; $\mathcal D_P\gets\varnothing$
\For{each delegation or revision attribution unit $u_t$ in $\mathcal B$}
    \State \textcolor{blue}{$(\texttt{FLAG}_t,\mathcal H_t)\gets M_{\theta_0}(I_{\mathrm{GRM}}(u_t))$} \Comment{Feedback attribution}
    \If{$\texttt{FLAG}_t=\texttt{true}$}
        \State $\mathcal D_P\gets\mathcal D_P\cup\{(s_t^{\mathrm{orch}},\mathcal H_t)\}$
    \EndIf
\EndFor
\For{each minibatch from $\operatorname{PairByContext}(\mathcal D_R,\mathcal D_P)$}
    \State $\mathcal L_R\gets\operatorname{RWR}(\pi_\theta;\text{successful-decision targets})$ \Comment{Outcome supervision}
    \For{each OPD record $(s,\mathcal H)$ in the minibatch}
        \State $a\sim\pi_\theta(\cdot\mid s)$ \Comment{Current-student generation}
        \State $p_k\gets\pi_\theta(\cdot\mid s,a_{<k})$; $q_k\gets\pi_{\theta_{\mathrm{SFT}}}^{+}(\cdot\mid s,\mathcal H,a_{<k})$
        \State \textcolor{blue}{$\ell_P(s)\gets\frac{1}{|a|}\sum_k D_{\mathrm{KL}}(p_k\|\operatorname{sg}[q_k])$} \Comment{Feedback-guided distillation}
    \EndFor
    \State $\mathcal L_P\gets\operatorname{Mean}_{s}[\ell_P(s)]$; $\theta\gets\operatorname{OptimizerStep}(\theta,\mathcal L_R+\lambda\mathcal L_P)$
\EndFor
\State \Return $\pi_\theta$
\end{algorithmic}
\end{algorithm}
Algorithm~\ref{alg:opod-training} summarizes how OPOD constructs and uses complementary outcome and process supervision. We first collect execution trajectories with the SFT policy, extracting decisions from successful trajectories as RWR targets and valid feedback hints from attribution units as OPD supervision. After pairing these records by decision context, we optimize the student on the resulting fixed dataset: RWR reinforces the stored successful decisions, while OPD guides actions newly generated by the current student at the stored contexts toward the frozen SFT teacher conditioned on the corresponding feedback hints. The OPD loss is computed as full-vocabulary reverse KL along these student-generated actions and combined with RWR to update the student. Loss normalization and supervision masks are specified in Section~\ref{app:opod-implementation}.
\endgroup

\clearpage
\begingroup
\raggedbottom
\setlength{\intextsep}{4pt}
\setlength{\abovecaptionskip}{3pt}
\setlength{\belowcaptionskip}{0pt}
\section{Additional Experiments and Analysis}
\label{app:additional-experiments}

We extend our evaluation through four additional research questions:
\begin{itemize}[leftmargin=1.5em, itemsep=0pt, topsep=0pt]
    \item \textbf{RQ5:} How sensitive is OPOD to the weight of process-level supervision? We vary $\lambda$ on GAIA2 (Section~\ref{app:hyperparameter-analysis}).
    \item \textbf{RQ6:} How do harness revision capabilities and communication protocol components contribute to proactive orchestration? We ablate these components individually (Section~\ref{app:detailed-ablation}).
    \item \textbf{RQ7:} Do savings in orchestration and execution offset the cost of proactive communication? We break down inference costs on GAIA and GAIA2 (Section~\ref{app:communication-cost}).
    \item \textbf{RQ8:} Which distillation objective is best suited to OPOD? We compare reverse KL, forward KL, and JSD on GAIA2 (Section~\ref{app:kl-direction}).
\end{itemize}

\subsection{Effect of OPD Weight (RQ5)}
\label{app:hyperparameter-analysis}

\begin{figure}[H]
\centering
\includegraphics[width=0.44\textwidth]{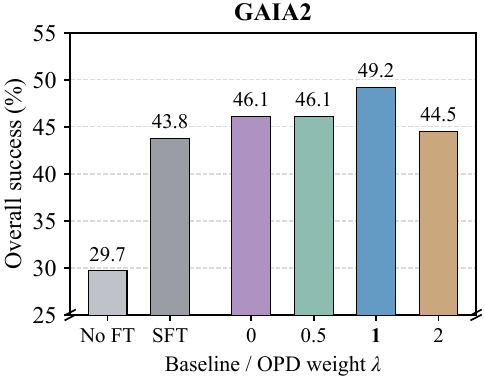}
\caption{Hyperparameter analysis of $\lambda$ on GAIA2.}
\label{fig:opd-weight}
\end{figure}

\noindent\textbf{Obs. \ding{190}: A moderate OPD weight balances process and outcome supervision.}
We vary $\lambda\in\{0,0.5,1,2\}$ on GAIA2 to assess sensitivity to the relative strength of process and outcome supervision, with $\lambda=0$ removing OPD. Figure~\ref{fig:opd-weight} shows that $\lambda=1$ achieves the highest success among the tested settings, suggesting a favorable balance between the two signals. Smaller weights may underuse process feedback, while larger weights may overemphasize the self-teacher relative to terminal rewards. Across the three nonzero settings, success ranges from 44.5\% to 49.2\%: a spread of 4.7 percentage points over a fourfold change in weight. This indicates moderate sensitivity within the tested range and supports $\lambda=1$ as our default.

\subsection{Harness Revision and Communication Ablations (RQ6)}
\label{app:detailed-ablation}
\label{app:protocol-ablation}

\begin{figure}[H]
\centering
\includegraphics[width=0.84\textwidth]{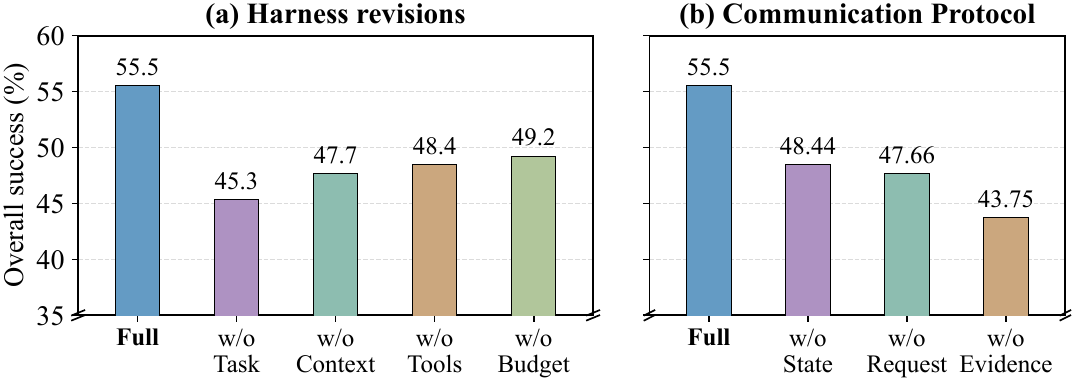}
\caption{Ablations on GAIA2: (a) harness revisions and (b) communication protocol components. Protocol ablation scores average two runs.}
\label{fig:detailed-ablation}
\label{fig:protocol-ablation}
\end{figure}

% Author confirmed on 2026-09-26: Full and all four harness ablations use
% Gemini-3-Flash as orchestrator and DeepSeek-V4-Flash as execution backbone.
% All four ablation counts are recorded in pic/analysis/gaia2_detailed_ablation_data.json.

\subsubsection*{Harness Revision Ablation}
\noindent\textbf{Obs. \ding{191}: Task and resource revisions contribute to effective proactive orchestration.}
On GAIA2, we disable revisions to task scope, working context, tools, or execution budget one at a time, preserving each component's initial configuration and the other revision capabilities. The context ablation still allows local tool observations to accumulate. Figure~\ref{fig:detailed-ablation}(a) shows that removing any revision capability reduces task success. Task revision has the largest observed effect, with a 10.2-point drop, highlighting the value of adjusting delegated subgoals as requirements evolve. Disabling context, tool, or budget updates yields drops of 7.8, 7.1, and 6.3 points, respectively. The losses persist with the other three capabilities retained, indicating that task and resource revisions provide complementary support during execution.

\input{appendix/protocol_ablation}

\par\addvspace{10pt}
\subsection{Communication Cost Analysis (RQ7)}
\label{app:feedback-analysis}
\label{app:communication-cost}

\setlength{\intextsep}{8pt}
\begin{table}[H]
\centering
\small
\setlength{\tabcolsep}{2.5pt}
\setlength{\arrayrulewidth}{0.25pt}
\setlength{\belowcaptionskip}{6pt}
\renewcommand{\arraystretch}{1.11}
\caption{Cost breakdown on GAIA and GAIA2 (CA per task).}
\label{tab:feedback-statistics}
\label{tab:communication-cost}
\begin{tabular}{
L{\dimexpr0.15\textwidth-2\tabcolsep\relax}
L{\dimexpr0.21\textwidth-2\tabcolsep\relax}
C{\dimexpr0.18\textwidth-2\tabcolsep\relax}
C{\dimexpr0.20\textwidth-2\tabcolsep\relax}
C{\dimexpr0.15\textwidth-2\tabcolsep\relax}
C{\dimexpr0.10\textwidth-2\tabcolsep\relax}}
\toprule
\textbf{Benchmark} & \textbf{Method} & \textbf{Orchestration} & \textbf{Communication} & \textbf{Sub-agent} & \textbf{Total} \\
\midrule
& \aorchestra{} & 0.236 & --- & 0.390 & 0.625 \\
& Orchestra-o1 & 0.201 & --- & 0.690 & 0.891 \\
\cdashline{2-6}[4pt/2pt]
\rowcolor{arporow}\cellcolor{white}\multirow{-3}{*}{\textbf{GAIA}}
& \textbf{\ourmethod{}} & 0.151 & 0.193 & 0.213 & \textbf{0.557} \\
\midrule
& \aorchestra{} & 0.208 & --- & 1.001 & 1.209 \\
& Orchestra-o1 & 0.244 & --- & 0.875 & 1.120 \\
\cdashline{2-6}[4pt/2pt]
\rowcolor{arporow}\cellcolor{white}\multirow{-3}{*}{\textbf{GAIA2}}
& \textbf{\ourmethod{}} & 0.097 & 0.065 & 0.559 & \textbf{0.721} \\
\bottomrule
\end{tabular}
\end{table}

% Extend the filled observation marker beyond the ten numerals in pifont.
\noindent\textbf{Obs. \makebox[0.9em]{\makebox[0pt]{\ding{108}}\makebox[0pt]{\raisebox{0.25ex}{\textcolor{white}{\fontsize{5.5}{5.5}\selectfont 12}}}}: Lower orchestration and execution costs offset communication overhead.}
Table~\ref{tab:communication-cost} compares per-task costs using \modelname{Gemini-3-Flash} as the orchestrator and \modelname{DeepSeek-V3.2} as the executor on both benchmarks. We separate routine orchestration, proactive communication, and sub-agent execution into non-overlapping components, using estimates for GAIA communication. \ourmethod{} spends less on routine orchestration and sub-agent execution than both baselines. These savings exceed the expense of generating and handling feedback, yielding the lowest total inference cost on both benchmarks. The breakdown shows that proactive communication can support efficient execution even after its own overhead is included.

\par\addvspace{10pt}
\subsection{Distillation Objective Comparison (RQ8)}
\label{app:kl-direction}

\setlength{\intextsep}{8pt}
\begin{table}[H]
\centering
\small
\setlength{\tabcolsep}{2.5pt}
\setlength{\arrayrulewidth}{0.25pt}
\setlength{\belowcaptionskip}{6pt}
\renewcommand{\arraystretch}{1.11}
\caption{Distillation objective comparison on GAIA2. All values are success rates (\%). Reverse KL uses the full OPOD results from Table~\ref{tab:optimization}.}
\label{tab:kl-direction}
\begin{tabular}{
L{\dimexpr0.25\textwidth-2\tabcolsep\relax}
*{5}{C{\dimexpr0.15\textwidth-2\tabcolsep\relax}}}
\toprule
\textbf{Objective} & \textbf{Execution} & \textbf{Search} & \textbf{Adaptability} & \textbf{Time} & \textbf{Overall} \\
\midrule
% Author-provided Forward KL counts: 16/32, 23/32, 9/32, 7/32; total 55/128.
% Author-provided JSD counts: 13/32, 24/32, 12/32, 9/32; total 58/128.
% The unlabelled values 2 (Forward KL) and 1 (JSD) are not interpreted here.
Forward KL & \textbf{50.0} & 71.9 & 28.1 & 21.9 & 43.0 \\
JSD & 40.6 & 75.0 & \textbf{37.5} & 28.1 & 45.3 \\
\rowcolor{arporow}
\textbf{Reverse KL} & 43.8 & \textbf{84.4} & \textbf{37.5} & \textbf{31.3} & \textbf{49.2} \\
\bottomrule
\end{tabular}
\end{table}

\noindent\textbf{Obs. \makebox[0.9em]{\makebox[0pt]{\ding{108}}\makebox[0pt]{\raisebox{0.25ex}{\textcolor{white}{\fontsize{5.5}{5.5}\selectfont 13}}}}: Reverse KL best complements OPOD among the compared objectives.}
Our main experiments use reverse KL for the OPD loss, raising the question of whether this choice is well suited to OPOD. To examine this, we compare three formulations of the distillation loss: reverse KL, forward KL, and Jensen--Shannon divergence (JSD). Table~\ref{tab:kl-direction} reports their success rates on GAIA2. Reverse KL achieves the highest overall success, exceeding forward KL and JSD by 6.3 and 3.9 percentage points, respectively. These results support its use in our main joint-training configuration. One possible explanation lies in the local and retrospective nature of sub-agent feedback. Reverse KL penalizes student-favored actions assigned low probability by the feedback-augmented teacher. This targeted correction may complement RWR's reinforcement of successful actions.

\endgroup
\begingroup
\raggedbottom
\section{Case Studies}
\label{app:case-study}

We present five cases illustrating how proactive sub-agent feedback supports runtime adaptation and orchestrator optimization:
\begin{itemize}[leftmargin=1.5em, itemsep=0pt, topsep=0pt]
    \item \textbf{Tool adaptation:} Sub-agents proactively request tool revisions, prompting the orchestrator to adjust their tool configurations (Figure~\ref{fig:gaia1-execution-case}).
    \item \textbf{Task adaptation:} Sub-agents proactively request task revisions, prompting the orchestrator to adjust their task scope (Figure~\ref{fig:gaia2-email-case}).
    \item \textbf{Context allocation optimization:} OPOD supervises the orchestrator's context allocation at delegation (Figure~\ref{fig:opod-context-case}).
    \item \textbf{Tool configuration optimization:} OPOD supervises the orchestrator's tool configuration at delegation (Figure~\ref{fig:opod-archive-case}).
    \item \textbf{Revision policy optimization:} OPOD supervises how the orchestrator revises harnesses in response to sub-agent requests (Figure~\ref{fig:opod-domain-case}).
\end{itemize}

\subsection{Proactive Adaptation in Execution}
\label{app:execution-case}
\subsubsection{Tool Adaptation}
\input{appendix/cases/greenland_case}

\subsubsection{Task Adaptation}
\label{app:gaia2-execution-case}
\input{appendix/cases/gaia2_email_case}

\subsection{Process Supervision in Optimization}
\label{app:opod-case}
\input{appendix/cases/opod_cases}

\clearpage
\endgroup
\section{Proofs}
\label{app:theory}

\subsection{Proof of Proposition 1}

\paragraph{Proposition 1 (Restated).}
Suppose a sub-agent requires at least $H_{\min}$ steps to maintain coherent execution, while re-orchestration is permitted only after its lifecycle ends. If a mismatch occurs uniformly over horizon $H$, then the probability of intervention within a recovery window $\delta$ satisfies
\[
\Pr(D_H\leq\delta)\leq\min\left\{1,\frac{\delta}{H_{\min}}\right\}.
\]

\begin{proof}
We formalize the adaptation delay induced by lifecycle-level re-orchestration. Consider a sub-agent execution with a fixed maximum horizon $H$, during which the orchestrator cannot modify the active configuration until the sub-agent reaches its lifecycle boundary. Let $\tau\in[0,H)$ denote the time at which an execution-time mismatch emerges within the current sub-agent lifecycle. Under the asynchronous-arrival assumption in Proposition~\ref{prop:boundary-delay},
\begin{equation}
\tau\sim\operatorname{Unif}(0,H).
\end{equation}
Since re-orchestration is unavailable before the lifecycle boundary at $H$, the earliest admissible intervention occurs at time $H$. The resulting adaptation delay is therefore
\begin{equation}
D_H\triangleq H-\tau.
\end{equation}
For any $d\in[0,H]$, the cumulative distribution function of $D_H$ is
\begin{align}
F_{D_H}(d)
&=\Pr(D_H\leq d) \notag\\
&=\Pr(H-\tau\leq d) \notag\\
&=\Pr(\tau\geq H-d) \notag\\
&=\int_{H-d}^{H}\frac{1}{H}\,\mathrm{d}\tau \notag\\
&=\frac{d}{H}.
\end{align}
Together with $F_{D_H}(d)=0$ for $d<0$ and $F_{D_H}(d)=1$ for $d\geq H$, this yields
\begin{equation}
D_H\sim\operatorname{Unif}(0,H).
\end{equation}
The expected adaptation delay is thus
\begin{align}
\mathbb E[D_H]
&=\int_0^H d\cdot\frac{1}{H}\,\mathrm{d}d \notag\\
&=\frac{1}{H}\left[\frac{d^2}{2}\right]_0^H \notag\\
&=\frac{H}{2}.
\end{align}
For completeness, its second moment is
\begin{equation}
\mathbb E[D_H^2]
=\int_0^H d^2\cdot\frac{1}{H}\,\mathrm{d}d
=\frac{H^2}{3},
\end{equation}
which gives
\begin{equation}
\operatorname{Var}(D_H)
=\mathbb E[D_H^2]-\mathbb E[D_H]^2
=\frac{H^2}{3}-\frac{H^2}{4}
=\frac{H^2}{12}.
\end{equation}
Hence, lifecycle-level re-orchestration structurally ties both the magnitude and dispersion of adaptation delay to the sub-agent horizon $H$.

We next characterize timely intervention. Let $\delta>0$ denote the effective recovery window of a mismatch, such that an orchestration adjustment is considered timely only if it becomes available within $\delta$ steps after the mismatch emerges. From the cumulative distribution above,
\begin{equation}
\Pr(D_H\leq\delta)
=F_{D_H}(\delta)
=\min\left\{1,\frac{\delta}{H}\right\}.
\end{equation}
In particular, when $\delta<H$,
\begin{equation}
\Pr(D_H\leq\delta)=\frac{\delta}{H},
\qquad
\frac{\partial}{\partial H}\Pr(D_H\leq\delta)
=-\frac{\delta}{H^2}<0.
\end{equation}
Thus, within lifecycle-level re-orchestration, increasing the temporal extent of a sub-agent necessarily decreases the probability that an asynchronously occurring mismatch can receive an intervention within a fixed recovery window.

Now suppose that maintaining a meaningful temporally extended sub-agent execution requires $H\geq H_{\min}$. Among all lifecycle-level re-orchestration schemes satisfying this requirement, define the maximum attainable timely-intervention probability as
\begin{equation}
C_{\mathrm{life}}^\star(\delta)
\triangleq
\sup_{H\geq H_{\min}}\Pr(D_H\leq\delta).
\end{equation}
Since the timely-intervention probability is non-increasing in $H$, the supremum is attained at the smallest admissible horizon $H=H_{\min}$. Therefore,
\begin{equation}
C_{\mathrm{life}}^\star(\delta)
=\sup_{H\geq H_{\min}}\min\left\{1,\frac{\delta}{H}\right\}
=\min\left\{1,\frac{\delta}{H_{\min}}\right\}.
\end{equation}
Consequently, every lifecycle-level re-orchestration scheme that preserves the required sub-agent temporal horizon satisfies
\begin{equation}
\boxed{
\Pr(D_H\leq\delta)
\leq
\min\left\{1,\frac{\delta}{H_{\min}}\right\}.
}
\end{equation}

For $H_{\min}\gg\delta$, lifecycle-level re-orchestration has maximum timely-intervention probability
\begin{equation}
C_{\mathrm{life}}^\star(\delta)
=\frac{\delta}{H_{\min}}\ll 1.
\end{equation}
Equivalently, defining the unavoidable timely-intervention gap as
\begin{equation}
\Gamma_{\mathrm{life}}(\delta)
\triangleq 1-C_{\mathrm{life}}^\star(\delta),
\end{equation}
we obtain, for $\delta<H_{\min}$,
\begin{equation}
\boxed{
\Gamma_{\mathrm{life}}(\delta)
=1-\frac{\delta}{H_{\min}}>0.
}
\end{equation}
Importantly, this gap follows from the lifecycle-level intervention structure rather than from the quality of the orchestrator's decision policy: before the lifecycle boundary is reached, no policy can revise the active sub-agent configuration without an admissible intervention point.

\ourmethod{} removes this boundary-induced constraint by augmenting the sub-agent action space with proactive communication. Let $\tau_{\mathrm{comm}}$ denote the first local decision point after mismatch emergence at which the sub-agent policy chooses to communicate and request adaptation. Its intervention delay takes the form
\begin{equation}
D_{\mathrm{PORCH}}=\tau_{\mathrm{comm}}-\tau,
\end{equation}
rather than the residual lifecycle duration $D_H=H-\tau$. Therefore, $D_{\mathrm{PORCH}}$ is governed by mismatch detection and the sub-agent's execution and communication policy, and is not structurally constrained to scale with the predefined lifecycle horizon $H$. In particular, reducing the intervention timescale no longer requires reducing the horizon over which the same sub-agent can maintain a coherent local execution trajectory. This establishes the claimed temporal decoupling.
\end{proof}

\subsection{Proof of Proposition 2}

\paragraph{Proposition 2 (Restated).}
Under bounded attribution noise, terminal rewards propagate through an attribution horizon $O(H_{\mathrm{orch}})$ times longer than that of OPD feedback, thereby accumulating noise over a longer execution chain. Furthermore, when the bias of OPD feedback is bounded, there exists a weight $\lambda>0$ such that OPOD yields lower expected error relative to ideal process-level supervision than RWR alone.

\begin{proof}
We analyze the learning signal associated with an orchestration decision at round $t$ through a bias--variance decomposition. Throughout, we condition on the current decision context and model parameters, treating the ideal update, attribution horizons, and process-supervision bias as fixed. Let $g_t^\star$ denote the ideal decision-level orchestration update direction obtained by evaluating the current decision without downstream stochasticity or attribution error. Here, we represent ideal process-level supervision by $g_t^\star$ and measure error as the expected squared distance of the normalized update from this reference. The proof uses the following sufficient conditions: stochastic attribution perturbations are bounded, zero-mean martingale differences; RWR and OPD perturbations are conditionally uncorrelated; the OPD bias is bounded; and terminal supervision has non-zero variance.

We first compare the attribution horizons of outcome-level and process-level supervision. Let $H_{\mathrm{orch}}^{(t)}$ denote the number of remaining orchestration rounds after decision $t$, and let $H_{\mathrm{sub}}$ denote the characteristic execution horizon of a sub-agent. A terminal reward associated with decision $t$ is observed only after all subsequent orchestration rounds and their corresponding sub-agent executions have completed. Assuming that the remaining orchestration rounds have comparable sub-agent execution lengths, its attribution horizon scales as
\begin{equation}
L_t^R=\Theta((1+H_{\mathrm{orch}}^{(t)})H_{\mathrm{sub}}).
\end{equation}
By contrast, \ourmethod{} returns proactive feedback during the execution of the current sub-agent. When feedback arrives after a comparable fraction of the current execution, the corresponding OPD attribution horizon scales as
\begin{equation}
L_t^P=\Theta(H_{\mathrm{sub}}).
\end{equation}
Consequently,
\begin{equation}
\frac{L_t^R}{L_t^P}=\Theta\!\left(1+H_{\mathrm{orch}}^{(t)}\right),
\end{equation}
where the additive one accounts for the current execution. This gives an $O(H_{\mathrm{orch}})$ attribution-horizon ratio over a task with $H_{\mathrm{orch}}$ orchestration rounds under these comparable-length assumptions.

We next characterize horizon-induced stochasticity. Suppose each downstream sub-agent--environment interaction introduces a zero-mean attribution perturbation. For terminal rewards,
\begin{equation}
g_t^R=g_t^\star+\sum_{k=1}^{L_t^R}\xi_{t,k}^R,
\end{equation}
where
\begin{equation}
\mathbb E[\xi_{t,k}^R\mid\mathcal F_{t,k-1}]=0,
\qquad
\mathbb E[\|\xi_{t,k}^R\|^2\mid\mathcal F_{t,k-1}]\leq\sigma_R^2.
\end{equation}
Here, $\mathcal F_{t,k}$ denotes the execution history up to the $k$-th downstream interaction. Under the martingale-difference assumption, cross-step covariance terms vanish, yielding
\begin{align}
V_R
&\triangleq\mathbb E[\|g_t^R-g_t^\star\|^2] \notag\\
&=\mathbb E\!\left[\left\|\sum_{k=1}^{L_t^R}\xi_{t,k}^R\right\|^2\right] \notag\\
&\leq L_t^R\sigma_R^2.
\end{align}
The variance bound grows linearly with the attribution horizon.

The OPD signal is generated from proactive feedback returned within the current sub-agent execution. Because this feedback is converted into a natural-language gradient and then into a feedback-augmented proxy teacher, it may introduce systematic attribution bias. We therefore write
\begin{equation}
g_t^P=g_t^\star+b_t+\sum_{k=1}^{L_t^P}\xi_{t,k}^P,
\end{equation}
where $b_t$ denotes the bias induced by the process-level proxy supervision and $\|b_t\|\leq B$ for some finite $B$. Similarly, assume
\begin{equation}
\mathbb E[\xi_{t,k}^P\mid\mathcal G_{t,k-1}]=0,
\qquad
\mathbb E[\|\xi_{t,k}^P\|^2\mid\mathcal G_{t,k-1}]\leq\sigma_P^2.
\end{equation}
Then
\begin{equation}
V_P\triangleq\mathbb E[\|g_t^P-\mathbb E[g_t^P]\|^2]
\leq L_t^P\sigma_P^2.
\end{equation}
Thus, the variance bounds scale with the respective attribution horizons: up to $O(H_{\mathrm{orch}}H_{\mathrm{sub}})$ for terminal supervision and $O(H_{\mathrm{sub}})$ for process supervision, with the latter also introducing finite bias $b_t$.

We now analyze their combination. The OPOD objective
\begin{equation}
\mathcal L_{\mathrm{OPOD}}=\mathcal L_{\mathrm{RWR}}+\lambda\mathcal L_{\mathrm{OPD}}
\end{equation}
induces the raw update $g_t^{\mathrm{raw}}=g_t^R+\lambda g_t^P$. To compare the signals at a common scale, we analyze the normalized update
\begin{equation}
g_t^{\mathrm{joint}}=\frac{g_t^R+\lambda g_t^P}{1+\lambda}.
\end{equation}
This normalization preserves the raw update direction. Define
\begin{equation}
\alpha=\frac{\lambda}{1+\lambda}\in[0,1),
\qquad
g_t^{\mathrm{joint}}=(1-\alpha)g_t^R+\alpha g_t^P.
\end{equation}
Let
\begin{equation}
\varepsilon_t^R=g_t^R-g_t^\star,
\qquad
\varepsilon_t^P=g_t^P-g_t^\star-b_t.
\end{equation}
Then
\begin{equation}
g_t^{\mathrm{joint}}-g_t^\star=(1-\alpha)\varepsilon_t^R+\alpha\varepsilon_t^P+\alpha b_t.
\end{equation}
By assumption,
\begin{equation}
\mathbb E[\langle\varepsilon_t^R,\varepsilon_t^P\rangle]=0.
\end{equation}
Since both stochastic components are zero mean, their cross terms with $b_t$ also vanish in expectation. Therefore,
\begin{align}
\mathcal E_t(\alpha)
&\triangleq\mathbb E[\|g_t^{\mathrm{joint}}-g_t^\star\|^2] \notag\\
&=(1-\alpha)^2V_R+\alpha^2V_P+\alpha^2\|b_t\|^2.
\end{align}
Using terminal-reward RWR alone corresponds to $\alpha=0$ and gives $\mathcal E_t(0)=V_R$. Moreover,
\begin{equation}
\left.\frac{\partial\mathcal E_t(\alpha)}{\partial\alpha}\right|_{\alpha=0}=-2V_R<0,
\end{equation}
where the strict inequality follows from $V_R>0$. Since $\|b_t\|<\infty$ and $V_P<\infty$, continuity implies that some sufficiently small $\alpha>0$ satisfies $\mathcal E_t(\alpha)<\mathcal E_t(0)$. By the one-to-one mapping
\begin{equation}
\lambda=\frac{\alpha}{1-\alpha},
\end{equation}
there exists a non-zero $\lambda>0$ satisfying
\begin{equation}
\boxed{
\mathbb E[\|g_t^{\mathrm{joint}}-g_t^\star\|^2]
<\mathbb E[\|g_t^R-g_t^\star\|^2].
}
\end{equation}

Under the stated assumptions, a suitable positive OPD weight reduces the normalized update's expected squared error. The shorter feedback horizon limits the number of downstream interactions that can contribute attribution noise, while the bounded bias captures deviations of feedback-based supervision from the ideal update direction. Combining RWR and OPD preserves task-level outcome supervision while introducing local corrective signals, addressing sparse supervision and long-horizon attribution in orchestration-policy optimization.
\end{proof}

\clearpage
\section{Implementation Details}
\label{app:implementation}

\subsection{Baseline and Benchmark Details}

\subsubsection{Baseline Details}
We group inference-time baselines into \textit{single-agent systems}, \textit{offline orchestration}, and \textit{re-orchestration}, as in the main experiments. We evaluate them on GAIA, GAIA2, and SWE-Bench, and compare SFT, RWR, and DA-GRPO for orchestrator optimization on GAIA and GAIA2.

\paragraph{ReAct.}
ReAct~\citep{yao2023react} interleaves reasoning and actions: reasoning traces support planning and progress tracking, while actions obtain information from the environment to inform subsequent decisions. The full task is handled within one agent's interaction history, making it a single-agent baseline. We instantiate this reasoning--action loop with each benchmark's available tools.

\paragraph{Mini-SWE.}
Mini-SWE~\citep{yang2024sweagent} is a minimal software-engineering agent that uses shell commands and a linear message history to inspect, modify, and test code. Its planning and execution remain within a single agent, providing another single-agent baseline. We use its repository interaction on SWE-Bench and adapt its tool interface to GAIA and GAIA2 while retaining the same execution pattern.

\paragraph{Skill-MAS.}
Skill-MAS~\citep{lin2026skill} represents orchestration expertise as an evolvable meta-skill covering task decomposition, agent engineering, and workflow orchestration. At inference time, the skill guides the generation of a task-specific agent team and executable workflow before execution, corresponding to offline orchestration in our comparison. We evaluate its initial meta-skill and connect the generated executors to each benchmark's tools.

\paragraph{AutoAgents.}
AutoAgents~\citep{ijcai2024p3} constructs a task-specific agent team and initial execution plan during a drafting stage, specifying roles and tool assignments before execution. In our offline-orchestration comparison, we fix these roles and the generated workflow and connect tool-using agents to each benchmark environment. Individual agents retain the ability to reason over observations and plan local actions within their assigned tasks.

\paragraph{MAS$^2$.}
MAS$^2$~\citep{wang2026mas} uses a Generator--Implementer--Rectifier team to construct and adapt multi-agent systems. The Generator produces a workflow, the Implementer assigns model backbones, and the Rectifier modifies the system when predefined conditions, such as execution failures or resource limits, are triggered. It represents re-orchestration driven by explicit triggers. We instantiate its workflow operators with benchmark-specific executors; on SWE-Bench, these operators generate, review, and select patch candidates.

\paragraph{Orchestra-o1.}
Orchestra-o1~\citep{zhang2026orchestrao1omnimodalagentorchestration} combines modality-aware task decomposition, sub-agent specialization, and dependency-aware parallel execution. The orchestrator assigns instructions, context, tools, and model backends to independent subtasks, then uses their returned summaries to plan the next delegation round or produce a final answer. Re-orchestration therefore occurs between rounds of sub-agent execution. We retain this staged delegation; on SWE-Bench, workers generate independent repair candidates for refinement and final selection.

\paragraph{\aorchestra{}.}
\aorchestra{}~\citep{ruan2026aorchestra} represents each sub-agent as a four-tuple of instruction, context, tools, and model, enabling on-demand executor specialization. Its orchestrator alternates between delegating to a newly instantiated sub-agent and finishing the task, using returned execution results to configure subsequent delegations. This places re-orchestration at sub-agent return boundaries. We preserve this delegation-and-return mechanism while connecting executors to each benchmark's tool environment.

\paragraph{SFT.}
SFT denotes standard supervised fine-tuning on offline expert trajectories, maximizing the likelihood of expert actions given their decision contexts. In our experiments, we fine-tune \modelname{Qwen3.5-9B} on orchestration decisions collected on TaskCraft for GAIA and on training scenarios for GAIA2. Each framework has its own SFT checkpoint, trained on structurally valid decisions from successful and unsuccessful expert trajectories, and initializes subsequent optimization.

\paragraph{RWR.}
RWR denotes reward-weighted regression, weighting action log-likelihoods by trajectory rewards. With binary success rewards, its objective reduces to supervised learning on successful trajectories, providing an outcome-supervised baseline. We collect trajectories from each framework's SFT policy and fine-tune its orchestrator on successful decision records. We evaluate RWR with Orchestra-o1, \aorchestra{}, and \ourmethod{} on GAIA and GAIA2.

\paragraph{DA-GRPO.}
DA-GRPO~\citep{zhang2026orchestrao1omnimodalagentorchestration} is the decision-aligned group relative policy optimization method introduced with Orchestra-o1. It scores candidate decisions against expert references using a rubric for format correctness, action validity, tool reasonableness, and decision quality, then forms group-relative advantages for policy optimization. This provides process-level supervision without executing each candidate in the environment. On GAIA and GAIA2, we train Orchestra-o1 using eight candidates per group and \modelname{DeepSeek-V3.2} for rubric scoring.

\subsubsection{Benchmark Details}
Table~\ref{tab:data-splits} summarizes the evaluated subsets, available tools, tool counts, and evaluation procedures for the three benchmarks.

\begin{table}[H]
\centering
\small
\setlength{\tabcolsep}{4pt}
\renewcommand{\arraystretch}{1.15}
\setlength{\belowcaptionskip}{6pt}
\caption{Benchmark subsets, available tools, and evaluation procedures.}
\label{tab:data-splits}
\begin{tabular}{@{}>{\raggedright\arraybackslash}m{0.115\textwidth}>{\raggedright\arraybackslash}m{0.14\textwidth}>{\raggedright\arraybackslash}m{0.42\textwidth}>{\centering\arraybackslash}m{0.085\textwidth}>{\raggedright\arraybackslash}m{\dimexpr0.24\textwidth-8\tabcolsep\relax}@{}}
\toprule
\textbf{Benchmark} & \textbf{Evaluation subset} & \textbf{Tools} & \textbf{Tool count} & \textbf{Evaluation} \\
\midrule
GAIA & 2023 validation; 165 tasks & \texttt{GoogleSearchAction}, \texttt{ExtractUrlContentAction}, \texttt{ExecuteCodeAction}, \texttt{ImageAnalysisAction}, \texttt{ParseAudioAction}. & 5 & Answer matching + semantic check. \\
\midrule
GAIA2 & Mini subset; 128 scenarios & Examples: \texttt{Emails\_\_search\_emails}, \texttt{Calendar\_\_add\_calendar\_event}, \texttt{Shopping\_\_add\_to\_cart}, \texttt{SystemApp\_\_wait\_for\_notification}, \ldots & 101 & ARE scenario validation. \\
\midrule
SWE-Bench Verified & Fixed test subset; 100 issues & \texttt{open}, \texttt{scroll\_up}, \texttt{scroll\_down}, \texttt{goto}, \texttt{edit}, \texttt{str\_replace}, \texttt{insert}, \texttt{create}, \texttt{delete\_lines}, \texttt{search\_file}, \texttt{search\_dir}, \texttt{find\_file}, \texttt{state}, and \texttt{bash}. & 14 & Official patch evaluation. \\
\bottomrule
\end{tabular}
\end{table}

% Sources: gaia1_paper_handoff_20260923/EXPERIMENT_DETAILS_zh.md,
% gaia2_paper_details_20260923_lite/EXPERIMENT_DETAILS.md and scenarios/,
% swebench_methods_prompts_20260923/README.md and example_configs/.
% Tool counts: GAIA five-tool registry; GAIA2 101 unique App__tool entries
% in prompts/full_vs_harness/Full_main_first.txt, also confirmed in all 161
% observed main_agent.json tool catalogs in the GAIA2 handoff archive;
% SWE ACIToolManager.commands has 13 entries plus Bash fallback (mini: Bash).
\paragraph{GAIA.}
GAIA~\citep{mialon2024gaia} evaluates general-purpose assistants on real-world questions that combine multi-step reasoning, web research, tool use, and multimodal information processing. We evaluate the full 2023 validation set of 165 tasks, comprising 53 Level-1, 86 Level-2, and 26 Level-3 tasks, with task-provided attachments available to the agents. The evaluation implementation combines normalized numerical, list, and string matching against reference answers with an LLM-based semantic-equivalence fallback. Overall success is computed across all 165 tasks, weighted by task count across levels. For orchestrator training, we follow \aorchestra{}~\citep{ruan2026aorchestra} and use its released training set of 2,339 unique TaskCraft tasks.

\paragraph{GAIA2.}
GAIA2~\citep{froger2026gaia} benchmarks digital assistants on completing user requests in dynamic, asynchronous application environments. It is built on Meta Agents Research Environments (ARE)~\citep{froger2025arescalingagentenvironments}, which simulates stateful applications and delivers user messages and environmental events during execution. Tasks require agents to retrieve information and perform operations across applications such as email, calendars, messaging, and shopping, while adapting to incoming updates and temporal constraints. We evaluate four capability-specific subsets:
\begin{itemize}[leftmargin=1.5em,itemsep=2pt,topsep=3pt]
    \item \textbf{Execution:} Following multi-step instructions and producing the required application state changes, such as sending messages or updating records.
    \item \textbf{Search:} Locating and combining information distributed across applications, such as emails, contacts, and files, to answer a query or identify the correct entities for subsequent actions.
    \item \textbf{Adaptability:} Adjusting ongoing task execution in response to asynchronous user updates or environmental events that change the requirements or relevant information.
    \item \textbf{Time:} Performing actions within specified time windows or deadlines, including deciding when to act and when to wait for future notifications.
\end{itemize}
The evaluation subset contains 128 mini scenarios, with 32 per category. Training uses the remaining 512 scenarios from these four categories, with 128 per category. The available tool catalog contains 101 tools, including the examples in Table~\ref{tab:data-splits}. Each scenario initializes its own application state; actions and notifications within the scenario preserve preceding state changes. ARE's scenario validator checks required events and outcomes using task-specific rules and semantic checks. Overall success is the fraction of valid scenarios among the 128 evaluation tasks, with timeouts and missing outputs counted as unsuccessful.

\paragraph{SWE-Bench Verified.}
SWE-Bench~\citep{jimenez2024swebench} evaluates software-engineering agents on resolving real GitHub issues. Following \aorchestra{}~\citep{ruan2026aorchestra}, we evaluate a fixed subset of 100 issues from the Verified test split and adopt its containerized repository execution environment and patch-evaluation setup.

\subsection{Training Configuration}
\input{appendix/training_implementation}

\clearpage
\section{Prompts and Communication Protocol}
\label{app:prompts-protocol}

\input{appendix/core_prompts}

\raggedbottom
\section{Limitations and Future Work}
\label{app:limitations}

\noindent\textbf{(1)} We currently optimize only the orchestrator, while keeping sub-agent model parameters and communication rules fixed. This allows the orchestrator to dynamically create and configure sub-agents for different tasks. Our experiments show that this design works across multiple execution models. Strong models can often recognize when their current strategy needs adjustment and proactively inform the orchestrator. We consider learning a general proactive-communication capability that transfers across models, tasks, and dynamically generated sub-agents a promising direction.

\noindent\textbf{(2)} Our work focuses on runtime orchestration adjustments through proactive sub-agent feedback. In the complex interaction tasks we study, we observe that non-concurrent orchestration baselines outperform the concurrent baseline, motivating our choice to use a single active sub-agent at a time. However, concurrent execution still offers potential efficiency gains for independent subtasks. Future work could coordinate feedback and revision requests across concurrent sub-agents to improve execution efficiency while maintaining global task consistency.

\section{Ethics Statement}
\label{app:ethics}

This work studies proactive coordination and policy optimization for language-model agents using existing benchmarks. We do not conduct human-subject studies. Agents that revise their goals, tools, context, or execution budgets can introduce risks when deployed with access to external systems, including unintended actions, privacy exposure, and misuse of granted permissions. Proactive revision requests should therefore remain subject to explicit authorization boundaries, with human oversight for consequential actions. Deployment in external systems requires further evaluation of action safety, permission handling, and privacy protection. Any release of code or execution traces should respect applicable licenses and data-use terms and exclude credentials and sensitive information. We acknowledge and adhere to the ICLR Code of Ethics.

\end{document}

%% file: math_commands.tex
\usepackage{amsmath,amsfonts,bm}

\def\eqref#1{equation~\ref{#1}}
\def\1{\bm{1}}

\DeclareMathAlphabet{\mathsfit}{\encodingdefault}{\sfdefault}{m}{sl}
\SetMathAlphabet{\mathsfit}{bold}{\encodingdefault}{\sfdefault}{bx}{n}

%% file: pic/analysis/gaia2_dynamicity_layout.tex
\def\dynamicitygraphic{\includegraphics[width=0.4\textwidth,trim=0bp 25.680000bp 0bp 0.000000bp]{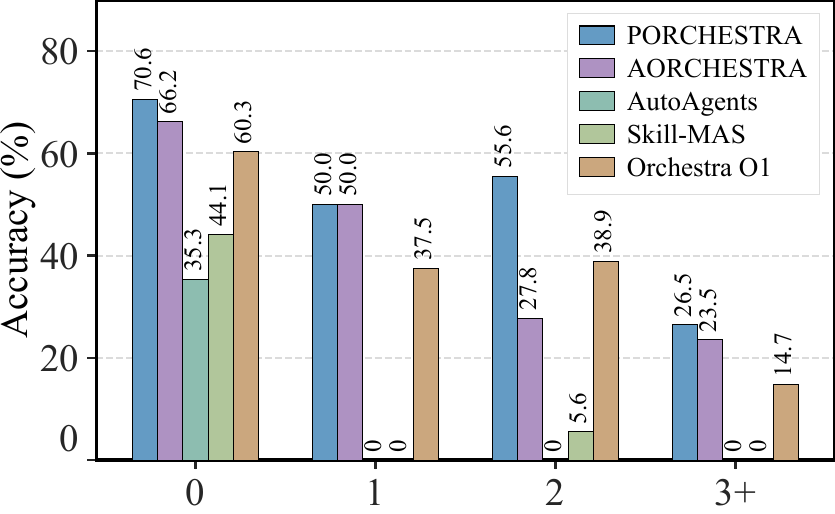}}

%% file: appendix/protocol_ablation.tex
% Source: experiments/protocol_ablation/gaia2_sae_ablation.json.
% Figure 6(b) shows protocol ablation means alongside the Full reference.
\subsubsection*{Communication Protocol Ablation}
\noindent\textbf{Obs. \makebox[0.9em]{\makebox[0pt]{\ding{108}}\makebox[0pt]{\raisebox{0.25ex}{\textcolor{white}{\fontsize{5.5}{5.5}\selectfont 11}}}}: Effective orchestration relies on all three communication components.}
We remove State, Request, or Evidence from the communication protocol one at a time on GAIA2 and average each ablation over two runs. Figure~\ref{fig:protocol-ablation}(b) shows that every removal reduces success relative to the full protocol. Without State, the orchestrator loses the sub-agent's summary of progress, uncertainty, and blockage, making it harder to understand why intervention is needed. Without Request, the orchestrator must infer the direction of adjustment from the remaining signals, without an explicit proposal from the sub-agent carrying out the task. Removing Evidence causes the largest drop: the orchestrator lacks the supporting tool observations needed to assess the reliability of the sub-agent's report and justify an update. These results highlight the complementary roles of execution awareness, adjustment guidance, and reliable supporting evidence in proactive orchestration.

%% file: appendix/cases/greenland_case.tex
% Source: experiments/case_studies/gaia1_dpv32_revise_cases_20260923/
% 3ff6b7a9-a5bd-4412-ad92-0cd0d45c0fee_20260701_162420.json.
% Machine-readable excerpts are extracted by scripts/extract_greenland_case.py.
% Source metadata: Figure~\ref{fig:gaia1-execution-case} presents a GAIA task requiring historical information retrieval. The orchestrator is \modelname{gemini-3-flash-preview}, and the sub-agent is \modelname{deepseek-v3.2}. The task ID is \texttt{3ff6b7a9-a5bd-4412-ad92-0cd0d45c0fee}. Action names and key inputs follow the log; long observations and communication messages are condensed, and repeated calls are omitted. Entry numbers refer to the original sub-agent trace, including communication actions.

\begingroup
\input{appendix/cases/case_style}
\begin{tcolorbox}[breakable,colback=white,colframe=black!35,arc=1mm,
title={GAIA case: tool adaptation},
title after break={GAIA case (continued)},
colbacktitle=arporow,coltitle=black,fonttitle=\bfseries,
boxsep=4pt,left=5pt,right=5pt,top=5pt,bottom=5pt]
\small
\setlength{\parskip}{4pt}
\casefield{User task}{\textit{The longest-lived vertebrate is named after an island. According to Wikipedia as of January 1, 2021, what is the 2020 estimated population of that island, to the nearest thousand?}}

\casestep{myblue}{myblue}{Communication}{1. Orchestrator: Initial delegation}{%
\casefield{Action}{\casecomm{delegate\_task}}
\casefield{Task input}{Identify the longest-lived vertebrate and its namesake island, then find the island's 2020 population estimate in Wikipedia as of January 1, 2021. The instruction suggests the Greenland shark and Greenland as likely targets.}
\casefield{Assigned tools}{\texttt{GoogleSearchAction}, \texttt{ExtractUrlContentAction}.}
}

\casestep{mypurple}{caseenv}{Environment interaction}{2. Sub-agent: Search}{%
\casefield{Action (entry 1)}{\casetool{GoogleSearchAction}}
\casefield{Input query}{\texttt{longest-lived vertebrate named after an island}}
\casefield{Observation}{The result identifies the Greenland shark as the longest-living vertebrate.}
\casefield{Action (entry 2)}{\casetool{GoogleSearchAction}}
\casefield{Input query}{\texttt{Greenland population 2020 Wikipedia January 1 2021}}
\casefield{Observation excerpt}{``As of 1 January \textbf{2024} the resident population of Greenland was estimated to be 56,699, an increase of 90 (0.2\%) from the previous year.''}
This search result does not meet the requested year and historical-source constraints.
}

\textcolor{black!60}{\textit{Repeated search attempts in entries 3--24 are omitted. The sub-agent tries date constraints, Wikipedia site restrictions, and infobox or revision-history terms before the next displayed call.}}

\casestep{mypurple}{caseenv}{Environment interaction}{3. Sub-agent: Page inspection}{%
\casefield{Action (entry 25)}{\casetool{ExtractUrlContentAction}}
\casefield{Input URL}{\url{https://en.wikipedia.org/wiki/Greenland}}
\casefield{Input \texttt{browse\_query}}{\texttt{2020 population estimate}}
\casefield{Observation}{The extractor reports that the \textbf{2020 estimate is not explicitly provided}. Its answer lists a 2025 estimate of 56,831 and a 2022 population of 56,583, concluding that it cannot provide the requested 2020 estimate from the available content.}
\casefield{Further search (entry 26)}{Another date-specific \casetool{GoogleSearchAction} returns ``No good Google Search Result was found.''}
}

\casestep{mypurple}{myblue}{Communication to Orchestrator}{4. Sub-agent: Revision request}{%
\casefield{Action (entry 27)}{\casecomm{revise}}
\casefield{\textcolor{myblue}{Status}}{Greenland has been identified, but the historical population estimate remains unresolved.}
\casefield{\textcolor{myblue}{Request}}{Provide access to \textbf{historical Wikipedia revisions} or archived versions around January 1, 2021.}
\casefield{\textcolor{myblue}{Evidence}}{Searches have not located the required snapshot, and extraction from the current page returns more recent population figures.}
}

\casestep{myblue}{myblue}{Communication to Sub-agent}{5. Orchestrator: Harness update}{%
\casefield{Action}{\casecomm{revise\_task}}
\casefield{Task update}{Keep the objective; retrieve the Wikipedia revision just before January 1, 2021.}
\casefield{Tool update}{Retain search and URL extraction; add \casetool{ExecuteCodeAction}.}
\casefield{Context update}{Supply a Python example for querying the MediaWiki revision API and reading population-related content.}
\casefield{Budget input}{\texttt{step\_budget = 15}.}
\casefield{Instruction excerpt}{``Use ExecuteCodeAction to query the MediaWiki API for the Greenland page revision just before 2021-01-01T00:00:00Z, retrieve its content, locate the 2020 population estimate inside the infobox, and report that figure to the nearest thousand.''}
}

\casestep{mypurple}{caseenv}{Environment interaction}{6. Sub-agent: Historical query}{%
\casefield{Action (entry 28)}{\casetool{ExecuteCodeAction}}
{\raggedright\casefield{Input}{MediaWiki query: \texttt{titles=Greenland; rvstart=2021-01-01T00:00:00Z; rvdir=older; rvlimit=1}.}}
\casefield{Output}{Population estimate: \textcolor{caseenv}{\textbf{56,081}}; rounded to the nearest thousand: \textcolor{caseenv}{\textbf{56,000}}.}
}

\casestep{mypurple}{myblue}{Communication to Orchestrator}{7. Sub-agent: Result return}{%
\casefield{Action (entry 29)}{\casecomm{finish}; return the population estimate and its supporting tool output.}
}
\casestep{myblue}{myblue}{Final answer}{8. Orchestrator: Completion}{%
\casefield{Action}{\casecomm{complete}, with \textcolor{caseenv}{\texttt{answer = "56,000"}}.}
\casefield{Recorded evaluation}{\textcolor{caseenv}{\texttt{success = true}}.}
}
\end{tcolorbox}
\captionof{figure}{Execution trace of proactive harness adaptation on GAIA.}
\label{fig:gaia1-execution-case}
\endgroup

Figure~\ref{fig:gaia1-execution-case} shows how execution feedback guides tool and context adaptation. Search and current-page extraction identify Greenland but fail to retrieve the required historical population estimate. The sub-agent requests access to historical revisions, prompting the orchestrator to add \texttt{ExecuteCodeAction} and MediaWiki API guidance. The same worker resumes with its accumulated search history, retrieves 56,081, and returns the correctly rounded answer of 56,000.

% Author note: the stored Python input contains doubled backslashes in parsing literals.
% The tool output is reproduced from the log; executor-side normalization has not been verified.
% This figure displays the exact API and execution parameters, not a reconstructed runnable script.

%% file: appendix/cases/case_style.tex
% Load within a group so the case macros remain local to each transcript.
\setlength{\belowcaptionskip}{8pt}
\definecolor{caseenv}{HTML}{287060}
\definecolor{casebaseline}{HTML}{707070}
\newcommand{\casefield}[2]{\par\noindent\textcolor{black!70}{\textbf{#1:}} #2\par}
\newcommand{\casecomm}[1]{\textcolor{myblue}{\texttt{#1}}}
\newcommand{\casetool}[1]{\textcolor{caseenv}{\texttt{#1}}}
% Role color, interaction color, interaction label, step heading, and step content.
\newcommand{\casestep}[5]{%
\par\noindent\begin{minipage}{\linewidth}
\small\setlength{\parskip}{2.5pt}\setlength{\fboxsep}{2pt}
\noindent\colorbox{#2!8}{\parbox{\dimexpr\linewidth-2\fboxsep\relax}{%
\textcolor{#1}{\textbf{#4}}\hfill\textcolor{#2}{\footnotesize\textbf{#3}}}}\par
{\color{#2!45}\hrule height 0.4pt}\vspace{2pt}
#5\end{minipage}\par\vspace{5pt}}

%% file: appendix/cases/gaia2_email_case.tex
% Evidence: appendix/cases/gaia2_email/evidence.json, generated by
% scripts/extract_gaia2_email_case.py from result histories and POrchestra LLM calls.
% Source metadata: Figure~\ref{fig:gaia2-email-case} presents a dynamic GAIA2 task with conditional follow-up actions. The scenario is \texttt{scenario\_universe\_22\_vhe87i}, using \modelname{gemini-3-flash-preview} for orchestration and \modelname{deepseek-v4-flash} for execution. Three successive notifications lead to three revisions of the same active sub-agent. Key action inputs and observed messages are retained; communication is condensed, and formatting retries and routine lookups are omitted.

\begingroup
\input{appendix/cases/case_style}
\begin{tcolorbox}[breakable,colback=white,colframe=black!35,arc=1mm,
title={GAIA2 case: task adaptation},
title after break={GAIA2 case (continued)},
colbacktitle=arporow,coltitle=black,fonttitle=\bfseries,
boxsep=4pt,left=5pt,right=5pt,top=5pt,bottom=5pt]
\small
\setlength{\parskip}{4pt}
\casefield{User task (condensed)}{Schedule ``Exclusive Wine Bar'' with Deniz Öztürk for the upcoming Sunday, 5--7 PM, tagged as work; send a ``Reservation Confirmed'' email and notify the user. If Deniz confirms within three minutes of the email, book the specified premium cab for Sunday at 4:30 PM and send a confirmation message to yesterday's training attendees in a new ``Client Confirmation'' group. Otherwise, take no further action.}

\casestep{myblue}{myblue}{Communication}{1. Orchestrator: Initial delegation}{%
\casefield{Action}{\casecomm{delegate\_task}}
\casefield{Assigned subtask}{Look up Deniz, create the calendar event, send the reservation email, and report completion to the user.}
\casefield{Initial tool scope}{Contact search, calendar creation, email sending, and user notification. Waiting, cab booking, and group messaging are not yet assigned.}
}

\casestep{mypurple}{caseenv}{Environment interaction}{2. Sub-agent: Initial app actions}{%
\casefield{Action}{\casetool{Calendar\_\_\allowbreak add\_calendar\_event}}
\casefield{Key inputs}{Title: ``Exclusive Wine Bar''; start/end: \texttt{2024-10-20 17:00:00 / 19:00:00}; tag: \texttt{work}; attendee: Deniz Öztürk.}
\casefield{Action}{\casetool{Emails\_\_\allowbreak send\_email}}
\casefield{Key inputs}{Recipient: \texttt{deniz.ozturk@gaia2mail.com}; subject: ``Reservation Confirmed''; content confirms the Sunday booking.}
\casefield{Observed result}{The tools return a calendar-event ID and a sent-email ID. The sub-agent informs the user that the initial actions are complete.}
}

\casestep{mypurple}{myblue}{Communication to Orchestrator}{3. Sub-agent: First follow-up request}{%
\casefield{Trigger}{A new-email notification arrives from Nalani Jensen.}
\casefield{Action}{\casecomm{revise}, with \texttt{task\_scope.operation=follow\_up}.}
\casefield{\textcolor{myblue}{Status}}{The assigned setup is complete; the new notification has not yet been acted on.}
\casefield{\textcolor{myblue}{Request}}{Review the notification and decide whether it requires further action.}
\casefield{\textcolor{myblue}{Evidence}}{The notification sender, the previously created event and email IDs, and Deniz's contact address are passed to the orchestrator.}
}

\casestep{myblue}{myblue}{Communication to Sub-agent}{4. Orchestrator: Extend the task scope}{%
\casefield{Action}{\casecomm{revise\_task}}
\casefield{Task update}{Check Nalani's message and ignore it if unrelated. Monitor for Deniz's confirmation, then perform the cab and group-message branch if the condition is met.}
\casefield{Resource update}{Add inbox reading, notification waiting, cab booking, calendar lookup, and group-messaging capabilities to the same sub-agent.}
\casefield{Preserved constraint}{No cab booking or group creation without the required confirmation.}
}

\casestep{mypurple}{myblue}{Communication to Orchestrator}{5. Sub-agent: Second follow-up request}{%
\casefield{Observed context}{\casetool{Emails\_\_\allowbreak list\_emails}, with \texttt{folder\_name=INBOX}, returns Nalani's message: ``Hey Aurélien, I'll be there on Sunday. Nalani''. It concerns Nalani, not Deniz. After a notification wait, a new email arrives from Aurelia.}
\casefield{Action}{\casecomm{revise}, with \texttt{task\_scope.operation=follow\_up}.}
\casefield{\textcolor{myblue}{Request}}{Is the new message relevant to the conditional actions? Should the worker check the inbox again or finish without those actions?}
\casefield{\textcolor{myblue}{Evidence}}{Nalani's sender and message, the new notification, and the identifiers from the initial setup.}
}

\casestep{myblue}{myblue}{Communication to Sub-agent}{6. Orchestrator: Refine the confirmation check}{%
\casefield{Action}{\casecomm{revise\_task}}
\casefield{Task and context update}{Inspect Aurelia's message and the inbox to determine whether Deniz, or someone on her behalf, confirmed attendance. Retain the tools and conditional cab/group instructions.}
\casefield{Subsequent observation}{Aurelia's message says, ``Thank you for the confirmation, Aurélien. I will try to be there.'' While the worker waits, a notification arrives from \texttt{deniz.ozturk@gaia2mail.com}.}
}

\casestep{mypurple}{myblue}{Communication to Orchestrator}{7. Sub-agent: Third follow-up request}{%
\casefield{Action}{\casecomm{revise}, with \texttt{task\_scope.operation=follow\_up}.}
\casefield{\textcolor{myblue}{Status}}{The worker believes the waiting window has elapsed, but the new notification may change whether the conditional actions are required.}
\casefield{\textcolor{myblue}{Request}}{Read and evaluate the new email, or stop because of the deadline? The worker requests guidance before taking further side effects.}
\casefield{\textcolor{myblue}{Evidence}}{The new notification and the earlier email contents and identifiers are carried forward.}
\textcolor{black!60}{\textit{The worker reports its uncertainty before committing to the next branch.}}
}

\casestep{myblue}{myblue}{Communication to Sub-agent}{8. Orchestrator: Authorize conditional follow-up}{%
\casefield{Action}{\casecomm{revise\_task}}
\casefield{Instruction excerpt}{``Verify the email from Deniz Öztürk (received at 07:02:36). If she confirmed her attendance, book the premium cab for Oct 20 at 4:30 PM, create the `Client Confirmation' Messages group with the attendees of yesterday's `Staff Training: Wine and Food Pairing' event, and send them the message that the VIP confirmed.''}
\casefield{Resource update}{Keep the same worker and its inbox, cab, calendar, and messaging tools.}
}

\casestep{mypurple}{caseenv}{Environment interaction}{9. Sub-agent: Execute the confirmed branch}{%
\raggedright
\casefield{Inbox check}{\casetool{Emails\_\_\allowbreak list\_emails}: \texttt{folder\_name=INBOX, offset=0, limit=10}.}
\casefield{Observed email}{From Deniz, timestamp \texttt{07:02:36}: ``Thanks Aurélien, I'll be there on Sunday at 5 PM. Looking forward to it!''}
\casefield{Cab action}{\casetool{Cabs\_\_\allowbreak order\_ride}; \texttt{service\_type=Premium}, \texttt{ride\_time=2024-10-20 16:30:00}; route: 9012 Park Ave to 4567 Broadway, New York. The tool returns \textcolor{caseenv}{\textbf{BOOKED}}.}
\casefield{Group actions}{Look up yesterday's training attendees, then call \casetool{Messages\_\_\allowbreak create\_group\_conversation} with title ``Client Confirmation'' and the retrieved user IDs. Call \casetool{Messages\_\_\allowbreak send\_message\_to\_group\_conversation} with content ``the VIP has confirmed and preparations should begin''.}
\casefield{Completion}{\casetool{AgentUserInterface\_\_\allowbreak send\_message\_to\_user} reports the completed actions. The recorded scenario evaluation is \textcolor{caseenv}{\textbf{Valid}}.}
}

\casestep{casebaseline}{casebaseline}{Paired reference run}{AOrchestra: End-of-subtask coordination}{%
\casefield{Recorded behavior}{The first worker schedules the event, sends the email, requests a 180-second wait, and concludes that no qualifying confirmation arrived. A second worker checks the Sent folder and inbox, then reaches the same conclusion.}
\casefield{Recorded outcome}{\textbf{Invalid}. The evaluator identifies missing cab booking, group creation, and group messaging, as well as an email-operation count mismatch.}
\casefield{Execution structure}{\ourmethod{}: one worker and three revisions. \aorchestra{}: two workers and no within-worker revisions.}
}
\end{tcolorbox}

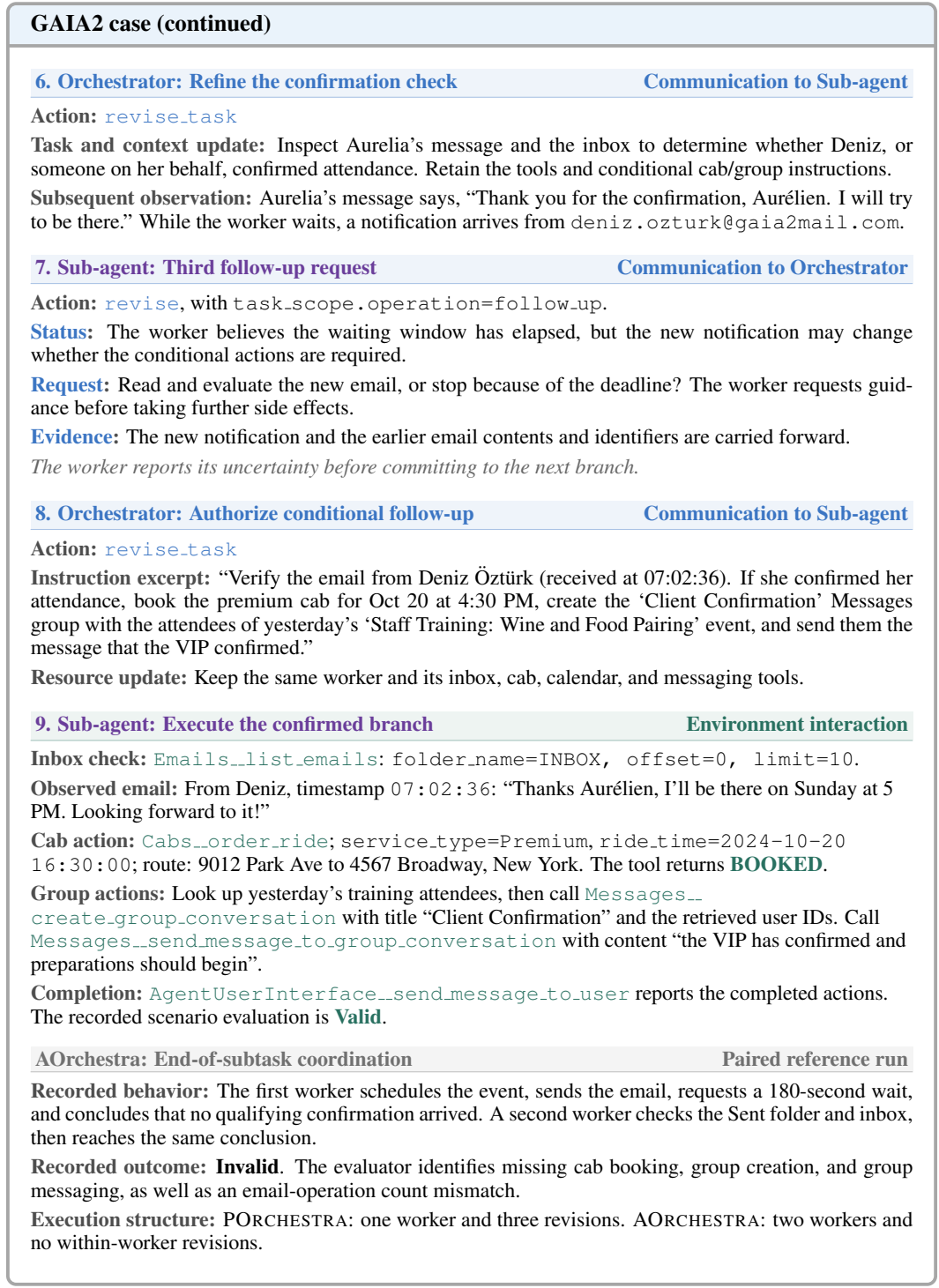
\captionof{figure}{Proactive task revisions under new notifications on GAIA2.}
\label{fig:gaia2-email-case}
\endgroup

\begingroup\interlinepenalty=10000
Figure~\ref{fig:gaia2-email-case} shows how proactive communication adapts an active delegation to new notifications. After completing the reservation setup, the worker reports successive messages and requests guidance on the conditional follow-up. Through three revisions, the orchestrator extends the task and tools, distinguishes unrelated replies from relevant confirmation, and directs verification of Deniz's response. The same worker then books the cab and sends the group message, carrying forward the earlier event, email, and contact information.
\par\endgroup

% Do not substitute oracle event times for agent-visible input.
% Main-agent deadline estimates differ from the timestamp of the initial send action.
% Source evidence is limited to visible notifications, model inputs/outputs, and result histories;
% validation rationale is used only to describe the final evaluation outcome.

%% file: appendix/cases/opod_cases.tex
% Sources: opd_case_23_5r4w6u_complete_traces_20260923.zip and
% gaia1_opd_two_examples_20260923.zip. Hints below reproduce saved outputs.
% Runtime revisions are observed actions, not self-teacher generations.

\subsubsection{Context Allocation Optimization}
% Source metadata: Figure~\ref{fig:opod-context-case} presents scenario \texttt{scenario\_universe\_23\_5r4w6u}. The displayed sequence numbers refer to the saved model-call log. The case connects a missing context handoff, a proactive request, and an attribution hint targeting the original delegation.

\begingroup
\input{appendix/cases/case_style}
\begin{tcolorbox}[breakable,colback=white,colframe=black!35,arc=1mm,
title={GAIA2 OPOD case: context allocation optimization},
title after break={GAIA2 OPOD case (continued)},
colbacktitle=arporow,coltitle=black,fonttitle=\bfseries,
boxsep=4pt,left=5pt,right=5pt,top=5pt,bottom=5pt]
\small\setlength{\parskip}{4pt}
\casefield{Relevant task objective}{Identify the postal code with the lowest property crime rate across the apartment catalog, save all lofts in that postal code, and report the saved properties to the user.}
\casestep{myblue}{myblue}{Communication}{1. Orchestrator: Delegation with incomplete context}{%
\casefield{Action}{\casecomm{delegate\_task}}
\casefield{Task input}{Query the remaining 12 postal codes, compare rates across \textbf{all catalog postal codes}, and save the lofts in the global minimum.}
\casefield{Context excerpt}{``Property crime rates have been queried for many zip codes already.''}
\casefield{Missing handoff}{The orchestrator already has rates for 18 postal codes, including \texttt{170 00 = 3.44}, but does not include those values in the delegation.}
}
\casestep{mypurple}{caseenv}{Environment interaction}{2. Sub-agent: Query the remaining postal codes}{%
\casefield{Actions}{\casetool{City\_\_get\_crime\_rate}, with API-count and limit checks.}
\casefield{Observed result}{All 12 remaining rates are collected. The lowest within this subset is \texttt{500 01 = 12.17}. The earlier 18 values remain unavailable in the worker's context.}
}
\casestep{mypurple}{myblue}{Communication to Orchestrator}{3. Sub-agent: Request the missing evidence}{%
\casefield{Action}{\casecomm{revise}; keep the task scope and request \texttt{context}.}
\casefield{\textcolor{myblue}{Status}}{The remaining queries are complete, but the global minimum cannot be confirmed.}
\casefield{\textcolor{myblue}{Request}}{Provide the previously queried postal codes and property crime rates.}
\casefield{\textcolor{myblue}{Evidence}}{The local minimum is \texttt{500 01 = 12.17}; selecting it without the earlier results could lead to saving lofts in the wrong postal code.}
}
\casestep{myblue}{myblue}{Communication to Sub-agent}{4. Orchestrator: Observed runtime repair}{%
\casefield{Action}{\casecomm{revise\_task}}
\casefield{Context update}{Supply all 30 rates, including \texttt{170 00 = 3.44}, and instruct the same worker to compare them before searching for and saving lofts.}
\casefield{Observed continuation}{The worker identifies \texttt{170 00}, saves \textbf{Prague 7 Loft}, and reports the result. The recorded scenario evaluation is \textcolor{caseenv}{\textbf{Valid}}.}
}
\casestep{mypurple}{mypurple}{Training signal}{5. Attribution: Repair the original delegation}{%
\casefield{Target}{The original delegation; $\texttt{FLAG}_t=\texttt{true}$.}
\casefield{Natural-language gradient $\mathcal H_t$}{``Include the previously queried zip codes and their property crime rates in the delegation context so the SubAgent can determine the global minimum without needing to request additional data.''}
\casefield{Training record}{The saved OPD sample pairs the original decision context with $\mathcal H_t$.}
}
\end{tcolorbox}

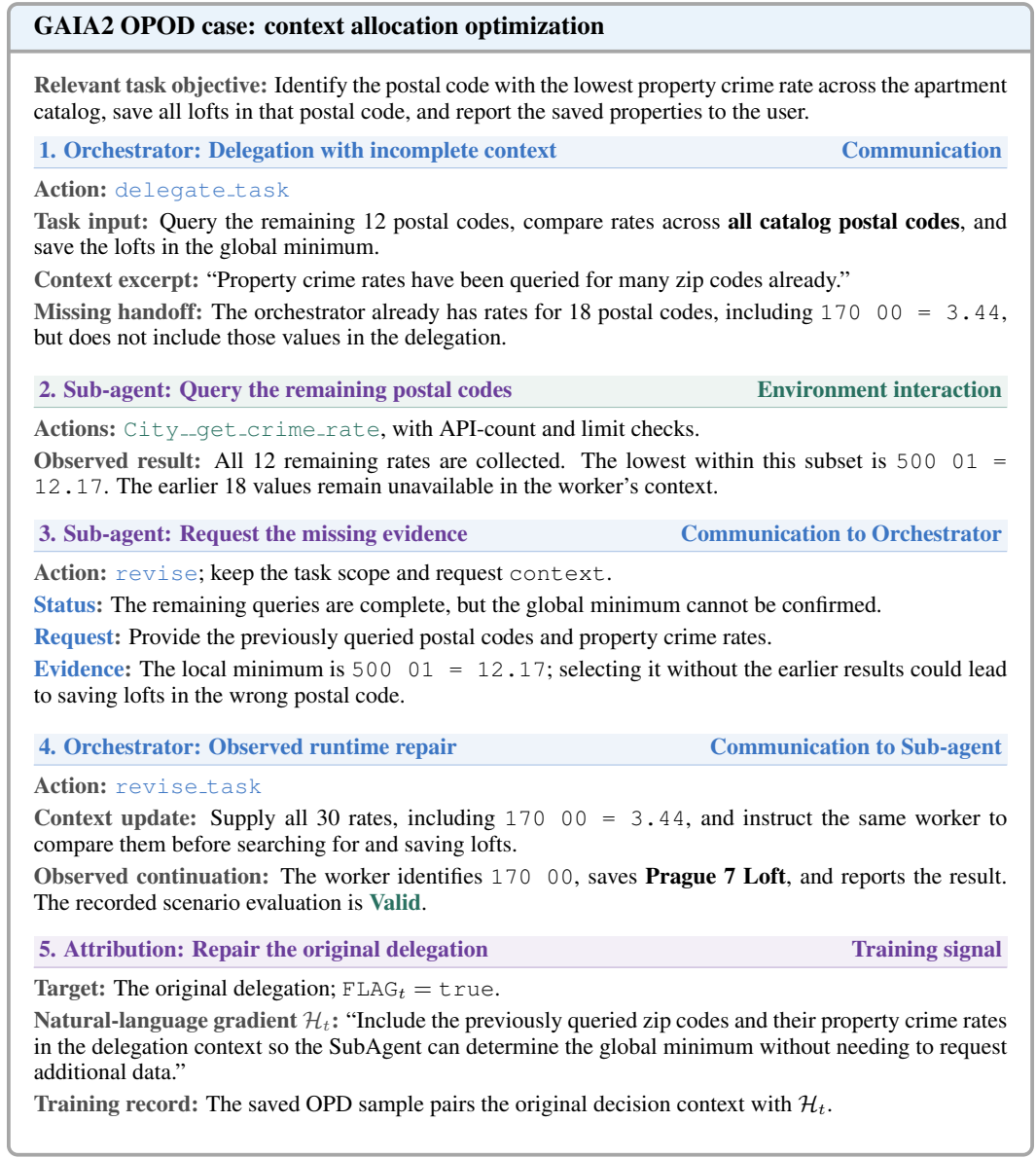
\captionof{figure}{A missing evidence handoff yields a context-repair gradient for OPOD.}
\label{fig:opod-context-case}
\endgroup

Figure~\ref{fig:opod-context-case} shows how OPOD generates supervision for context allocation at delegation. When constructing the sub-agent's context, the orchestrator omits previously collected crime rates, leaving the sub-agent unable to compare all postal codes after completing its queries. The proactive request identifies information that was available at decision time but not passed to the sub-agent. Runtime revision supplies the missing data, while OPOD attributes the omission to the original delegation and generates a gradient to include those results when creating the sub-agent. Thus, even a successful trajectory can provide supervision for improving delegation.

\subsubsection{Tool Configuration Optimization}
% Source metadata: Figure~\ref{fig:opod-archive-case} presents \texttt{taskcraft\_12822}, a TaskCraft sample used to train the GAIA1 orchestrator, with a Qwen-SFT main agent and a \modelname{DeepSeek-V3.2} sub-agent. The delegation already follows a failed live-page retrieval, making archival access relevant at decision time.

\begingroup
\input{appendix/cases/case_style}
\begin{tcolorbox}[breakable,colback=white,colframe=black!35,arc=1mm,
title={TaskCraft OPOD case: tool configuration optimization},
title after break={TaskCraft OPOD case (continued)},
colbacktitle=arporow,coltitle=black,fonttitle=\bfseries,
boxsep=4pt,left=5pt,right=5pt,top=5pt,bottom=5pt]
\small\setlength{\parskip}{4pt}
\casefield{User task (condensed)}{Find the update date of an article on \texttt{academiascribes.com} comparing Claude and ChatGPT for academic writing.}
\casestep{myblue}{myblue}{Communication}{1. Orchestrator: Delegate archive retrieval}{%
\casefield{Action}{\casecomm{delegate\_task}}
\casefield{Known context}{The suspected article URL returns \textbf{404}; the homepage appears repurposed.}
\casefield{Task input}{Search Wayback Machine archives, caches, and other pages for the article's update date.}
\casefield{Assigned tools}{\casetool{GoogleSearchAction}, \casetool{ExtractUrlContentAction}.}
}
\casestep{mypurple}{caseenv}{Environment interaction}{2. Sub-agent: Search and inspect the missing page}{%
\casefield{Actions}{Search the target URL with cache and archive terms; inspect the article using \casetool{ExtractUrlContentAction}.}
\casefield{Observed result}{The page returns \textbf{404 Not Found}. Archive searches yield no usable snapshot, and broader results concern other websites.}
}
\casestep{mypurple}{myblue}{Communication to Orchestrator}{3. Sub-agent: Request direct archival access}{%
\casefield{Action}{\casecomm{revise}; request a tool resource.}
\casefield{\textcolor{myblue}{Status}}{The article is unavailable, and search-based archive discovery has not found the required date.}
\casefield{\textcolor{myblue}{Request}}{Enable direct Wayback Machine API queries and retrieval of the archive URLs they return.}
\casefield{\textcolor{myblue}{Evidence}}{Searches return unrelated pages or no relevant results; direct archival queries may locate snapshots absent from search-engine results.}
}
\casestep{mypurple}{mypurple}{Training signal}{4. Attribution: Repair the delegated tool set}{%
\casefield{Target}{The original \casecomm{delegate\_task}; $\texttt{FLAG}_t=\texttt{true}$.}
\casefield{Natural-language gradient $\mathcal H_t$}{``Enable ExecuteCodeAction in the delegation to allow the SubAgent to programmatically query the Wayback Machine API or construct direct archive URLs, as GoogleSearchAction proved insufficient for locating the specific archive snapshot.''}
\casefield{Training record}{The original decision context and $\mathcal H_t$ appear in the saved OPD training sample.}
\casefield{Recorded task outcome}{\textbf{Timeout / Invalid}.}
}
\end{tcolorbox}

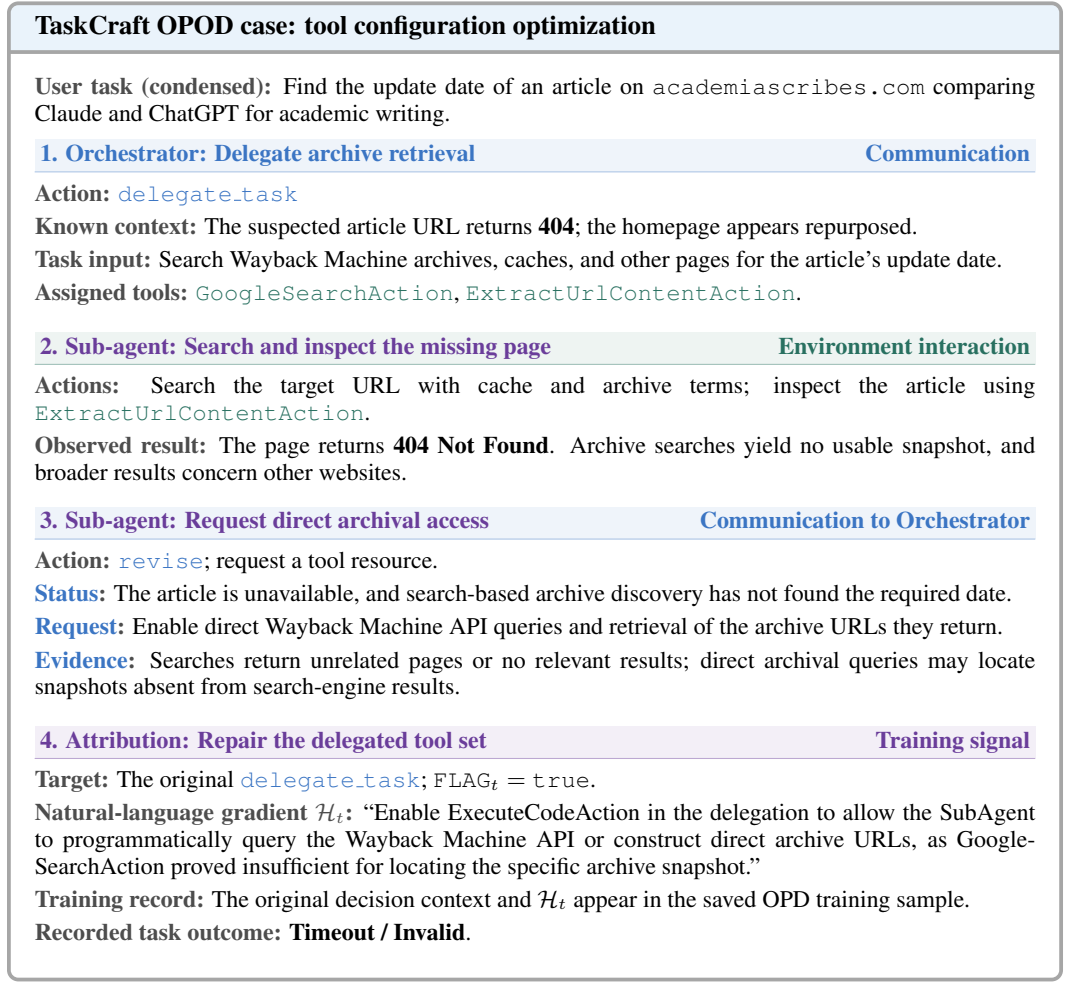
\captionof{figure}{A proactive tool request supports a gradient for archive-retrieval delegation.}
\label{fig:opod-archive-case}
\endgroup

Figure~\ref{fig:opod-archive-case} shows how OPOD generates supervision for tool configuration at delegation. When creating the sub-agent, the orchestrator assigns only search and extraction tools, failing to account for the need to retrieve the article's update date from archives after direct page access has failed. After archive searches fail, the sub-agent requests direct archival API access. OPOD attributes this resource mismatch to the delegation and generates a gradient to include \texttt{ExecuteCodeAction}, retaining the retrieval objective while improving the assigned tool set. This supervision encourages the orchestrator to configure tools suited to each sub-agent's task and known execution constraints.

\subsubsection{Revision Policy Optimization}
% Source metadata: Figure~\ref{fig:opod-domain-case} presents \texttt{taskcraft\_13294}, a TaskCraft sample used to train the GAIA1 orchestrator, with a Qwen-SFT main agent and a \modelname{DeepSeek-V4-Flash} sub-agent. Here the target of attribution is a revision that introduces an unsupported website hypothesis.

\begingroup
\input{appendix/cases/case_style}
\begin{tcolorbox}[breakable,colback=white,colframe=black!35,arc=1mm,
title={TaskCraft OPOD case: revision policy optimization},
title after break={TaskCraft OPOD case (continued)},
colbacktitle=arporow,coltitle=black,fonttitle=\bfseries,
boxsep=4pt,left=5pt,right=5pt,top=5pt,bottom=5pt]
\small\setlength{\parskip}{4pt}
\casefield{User task (condensed)}{Find an AI-writing-assistant comparison in an Academic Writing category and report its account of ChatGPT's development origins and current versions.}
\casestep{myblue}{myblue}{Communication to Sub-agent}{1. Orchestrator: Supply a speculative domain}{%
\casefield{Decision context}{The active worker has requested context and clues to locate the official website.}
\casefield{Action}{\casecomm{revise\_task}; retain search and URL extraction, with \texttt{step\_budget = 15}.}
\casefield{Context excerpt}{``The official website is likely `academic-writing.org' or a similar site. Please search specifically using \texttt{site:academic-writing.org} combined with keywords like `selecting', `AI writing assistant', or `ChatGPT'.''}
}
\casestep{mypurple}{caseenv}{Environment interaction}{2. Sub-agent: Follow the supplied lead}{%
\casefield{Action}{\casetool{GoogleSearchAction}}
\casefield{Input query}{\texttt{site:academic-writing.org "AI writing assistant" ChatGPT}}
\casefield{Observation}{``No good Google Search Result was found''. Further searches and page inspection follow.}
\casefield{Action}{\casetool{ExtractUrlContentAction}, with URL \url{https://academic-writing.org}.}
\casefield{Observation}{The extracted page describes a human essay-writing service and provides no relevant information about the requested AI-writing-assistant article.}
}
\casestep{mypurple}{mypurple}{Training signal}{3. Attribution: Correct the revision context}{%
\casefield{Target}{The original \casecomm{revise\_task}; $\texttt{FLAG}_t=\texttt{true}$.}
\casefield{Natural-language gradient $\mathcal H_t$}{``Avoid providing unverified specific domain names in the context during revisions. Instead, instruct the SubAgent to search for the article title and keywords to identify the correct official website, preventing wasted steps on incorrect domains.''}
\casefield{Training record}{The original decision context and $\mathcal H_t$ appear in the saved OPD training sample.}
\casefield{Recorded task outcome}{\textbf{Timeout / Invalid}.}
}
\end{tcolorbox}
\captionof{figure}{Execution evidence identifies unsupported context introduced by a revision.}
\label{fig:opod-domain-case}
\endgroup

Figure~\ref{fig:opod-domain-case} shows how OPOD also supervises the orchestrator's revision policy. In response to a request for help, the orchestrator supplies an unverified domain that directs the worker toward an irrelevant website. Subsequent search and extraction expose this mismatch. OPOD attributes the problem to the revision and generates a gradient to establish the website from article titles and keywords before supplying domain-specific guidance, providing supervision for how the orchestrator updates an active worker's context.

%% file: appendix/training_implementation.tex
% Sources: gaia1_paper_handoff_20260923.tar.gz and
% gaia2_paper_details_20260923_lite.zip; see the accompanying source notes.
\begingroup
\raggedbottom

\subsubsection{Base Model and Warm-Start Stage}
\label{app:sft-settings}
We initialize the trainable orchestrator from \modelname{Qwen3.5-9B} and perform SFT separately for Orchestra-o1, \aorchestra{}, and \ourmethod{} on each benchmark. Expert trajectories use \modelname{Gemini-3-Flash} as the orchestrator. For GAIA, we collect one pass over the 2,339 TaskCraft tasks with each of two execution backbones, \modelname{DeepSeek-V3.2} and \modelname{DeepSeek-V4-Flash}. For GAIA2, we collect four passes over the 512 training scenarios with \modelname{DeepSeek-V4-Flash} as the executor. Table~\ref{tab:sft-data} reports the decision records retained for SFT.

\begin{table}[H]
\centering
\small
\setlength{\tabcolsep}{5pt}
\setlength{\belowcaptionskip}{6pt}
\renewcommand{\arraystretch}{1.12}
% Center the multi-line caption locally for this table.
\makeatletter
\long\def\@makecaption#1#2{%
  \vskip\abovecaptionskip
  {\centering #1: #2\par}%
  \vskip\belowcaptionskip}
\makeatother
\caption{SFT data by benchmark and orchestration framework. Task pools and collection passes precede filtering; decision counts refer to the retained examples.}
\label{tab:sft-data}
\begin{tabular}{@{}clccc@{}}
\toprule
\textbf{Benchmark} & \textbf{Framework} & \textbf{Task pool} & \textbf{Passes} & \textbf{Train decisions} \\
\midrule
\multirow{3}{*}{GAIA} & Orchestra-o1 & 2,339 & 2 & 11,746 \\
& \aorchestra{} & 2,339 & 2 & 9,977 \\
& \ourmethod{} & 2,339 & 2 & 9,999 \\
\midrule
\multirow{3}{*}{GAIA2} & Orchestra-o1 & 512 & 4 & 5,276 \\
& \aorchestra{} & 512 & 4 & 4,087 \\
& \ourmethod{} & 512 & 4 & 4,539 \\
\bottomrule
\end{tabular}
\end{table}

SFT uses structurally valid orchestration decisions from both successful and unsuccessful expert trajectories. Subsequent optimization starts from the corresponding framework's SFT checkpoint; sub-agent models remain fixed.

\subsubsection{Trajectory Collection and Optimization Data}
\label{app:rollout-settings}
For RWR and OPOD, we collect trajectories from each framework's SFT policy using \modelname{DeepSeek-V4-Flash} as the executor, with one pass over the TaskCraft training set for GAIA and four passes over the training scenarios for GAIA2. Successful trajectories provide RWR targets, while OPOD additionally uses feedback from both successful and unsuccessful trajectories for OPD supervision. The two objectives are paired by decision context. During optimization, the stored contexts and hints remain fixed while the current student resamples actions for OPD. DA-GRPO samples and scores candidate actions at expert decision contexts.

\subsubsection{Optimization Hyperparameters}
\label{app:optimization-settings}
\begin{table}[H]
\centering
\small
\setlength{\tabcolsep}{5pt}
\setlength{\belowcaptionskip}{6pt}
\renewcommand{\arraystretch}{1.12}
\caption{Main SFT and joint OPOD training settings.}
\label{tab:training-configuration}
\begin{tabular}{@{}L{0.34\textwidth}L{0.28\textwidth}L{\dimexpr0.38\textwidth-4\tabcolsep\relax}@{}}
\toprule
\textbf{Parameter} & \textbf{SFT} & \textbf{OPOD} \\
\midrule
Parameter update & Full fine-tuning & LoRA, $r=32$, $\alpha=64$ \\
LoRA dropout & --- & 0.05 \\
Learning rate & $10^{-5}$ & $10^{-5}$ \\
Schedule / warmup ratio & Cosine / 0.05 & Cosine / 0.05 \\
Effective batch size & 16 & 16 \\
Per-device batch / accumulation & 1 / 2 & 1 / 2 \\
Epochs & 3 & 1 \\
Sequence-length setting & 32,768 tokens & 32,768 tokens \\
Generated action limit & --- & 1,024 tokens \\
Student sampling & --- & $T=1$, top-$p=0.95$, top-$k=50$ \\
Precision & BF16 & BF16; FP32 KL computation \\
Hardware & $8\times$ A800 80GB & $8\times$ A800 80GB \\
\bottomrule
\end{tabular}
\end{table}

Table~\ref{tab:training-configuration} summarizes the main training hyperparameters for SFT and OPOD. We implement SFT with LLaMA-Factory~\citep{zheng2024llamafactory} and OPOD with a custom trainer built on Hugging Face TRL~\citep{vonwerra2020trl}.

\paragraph{DA-GRPO.}
We initialize DA-GRPO from the corresponding Orchestra-o1 SFT checkpoint and train for one epoch, using batch size 24 and eight candidate actions per decision context. GAIA uses full fine-tuning with a learning rate of $5\times10^{-6}$, while GAIA2 uses LoRA with rank 32 and a learning rate of $10^{-5}$. Both settings use \modelname{DeepSeek-V3.2} as the rubric reward model.

\subsubsection{OPOD Loss Implementation}
\label{app:opod-implementation}
We use the frozen base \modelname{Qwen3.5-9B} model $M_{\theta_0}$, without task-specific SFT, to generate attribution flags and natural-language gradients, computing the OPD loss only for decisions with $\texttt{FLAG}=\texttt{true}$. To promote training stability, we fix the teacher to $M_{\theta_{\mathrm{SFT}}}$ throughout OPOD training, following on-policy self-distillation~\citep{zhao2026selfdistilled}. The student $M_\theta$ is initialized with $\theta=\theta_{\mathrm{SFT}}$ and is the only model updated during OPOD. The teacher receives feedback added to the original decision context. After the current student samples an action, the student and teacher score the same action under their respective contexts to compute token-level reverse KL over the full vocabulary, with gradients updating only the student. RWR computes cross-entropy on decisions from successful trajectories. Both losses average output-token losses within each decision, then across minibatch decisions with the corresponding supervision.

\subsection{Inference and Evaluation Configuration}
\label{app:inference-settings}
Table~\ref{tab:experiment-models} distinguishes the model settings for inference, training evaluation, and mechanism analysis. The API model identifiers are \texttt{gemini-3-flash-preview}, \texttt{deepseek-v3.2}, and \texttt{deepseek-v4-flash}. Single-agent baselines use the execution backbone directly.

\begin{table}[H]
\centering
\small
\setlength{\tabcolsep}{5pt}
\setlength{\belowcaptionskip}{6pt}
\renewcommand{\arraystretch}{1.12}
\caption{Model configurations for the main experiments and additional analyses.}
\label{tab:experiment-models}
\begin{tabular}{@{}>{\raggedright\arraybackslash}m{0.30\textwidth}>{\raggedright\arraybackslash}m{0.225\textwidth}>{\raggedright\arraybackslash}m{0.255\textwidth}>{\raggedright\arraybackslash}m{\dimexpr0.22\textwidth-6\tabcolsep\relax}@{}}
\toprule
\textbf{Experiment} & \textbf{Orchestrator} & \textbf{Executor} & \textbf{Evaluation} \\
\midrule
Inference comparison & \modelname{Gemini-3-Flash} & \modelname{Gemini-3-Flash},\newline \modelname{DeepSeek-V3.2},\newline \modelname{DeepSeek-V4-Flash} & GAIA, GAIA2, SWE-Bench \\
\midrule
Orchestrator optimization & \modelname{Qwen3.5-9B} & \modelname{DeepSeek-V4-Flash} & GAIA, GAIA2 \\
\midrule
Intervention / triggering / polling & \modelname{Gemini-3-Flash} & \modelname{DeepSeek-V4-Flash} & GAIA2 \\
\midrule
OPD-weight analysis & \modelname{Qwen3.5-9B} & \modelname{DeepSeek-V4-Flash} & GAIA2 \\
\midrule
Harness revision ablation & \modelname{Gemini-3-Flash} & \modelname{DeepSeek-V4-Flash} & GAIA2 \\
\bottomrule
\end{tabular}
\end{table}

\paragraph{Execution details.}
We serve \modelname{Qwen3.5-9B} with vLLM~\citep{kwon2023efficient} using a 32,768-token context window. Thinking is disabled for trained orchestrators and enabled for the untuned Qwen baseline on GAIA2. For GAIA and GAIA2, we set a total execution budget of 300 steps and allow the creation of up to 10 sub-agents, each with a default budget of 30 steps. For SWE-Bench, the total budget is 500 steps, with a default budget of 50 steps per sub-agent. Baselines use the same total execution budget on each benchmark.

\paragraph{Evaluation Protocol.}
For the results reported in Table~\ref{tab:optimization}, we evaluate the two best-performing methods in each task category twice and report their mean success rates; the remaining methods are evaluated once.

\paragraph{API access and cost units.}
All external model calls use the ChatAnywhere API\footnote{\url{https://chatanywhere.tech/}}. We report API costs in CA, with an approximate conversion of $1\,\mathrm{CA}\approx0.14\,\mathrm{USD}$.

\endgroup

%% file: appendix/core_prompts.tex
% Adapted from the GAIA, GAIA2, and SWE-Bench prompt handoff archives.
% Long runtime inputs and repeated examples are represented by placeholders.
% The delegation-scope restriction in the attribution source is intentionally omitted.
\begingroup
\raggedbottom

The following templates retain the core instructions, decision rules, and output formats, with placeholders for long runtime inputs and repeated examples. The highlighted block in the GAIA and GAIA2 main-agent prompts inserts $\mathcal H_t$ for the training-time self-teacher only.

\newtcolorbox{coreprompt}[2]{%
  breakable,colback=white,colframe=#2!50,arc=1mm,
  title={#1},title after break={#1 (continued)},
  colbacktitle=#2!10,coltitle=black,fonttitle=\bfseries,
  fontupper=\small\raggedright,boxsep=3pt,left=5pt,right=5pt,top=4pt,bottom=4pt,
  before skip=7pt,after skip=8pt}
\newcommand{\prompthead}[1]{\par\vspace{4pt}\noindent\textbf{#1}\par\nopagebreak}
\newcommand{\promptfield}[2]{\par\noindent\textbf{#1:} #2\par}
\newcommand{\promptslot}[1]{\textcolor{black!65}{\texttt{\textless#1\textgreater}}}
\newenvironment{promptrules}{\begin{itemize}[leftmargin=1.3em,itemsep=1pt,parsep=0pt,topsep=2pt]}{\end{itemize}}
\newenvironment{promptcode}{\par\vspace{2pt}\begingroup\ttfamily\small\setlength{\parindent}{0pt}\setlength{\parskip}{0pt}}{\par\endgroup}
\newcommand{\pc}[1]{\noindent #1\par}
\newcommand{\prompthint}{%
  \begin{tcolorbox}[colback=mypurple!6,colframe=mypurple!35,arc=0.5mm,
    boxsep=2pt,left=4pt,right=4pt,top=2pt,bottom=2pt,
    before skip=4pt,after skip=4pt]
  \textcolor{mypurple}{\textbf{Hint / Instruction [self-teacher only]}}\par
  $\mathcal H_t$: \promptslot{natural-language gradient for this decision}
  \end{tcolorbox}}

\subsection{Orchestrator Prompts}
\label{app:orchestrator-prompt}
\label{app:self-teacher-prompt}

\begin{coreprompt}{GAIA: Main Agent}{myblue}
You are the MainAgent in \ourmethod{}. Coordinate SubAgents to solve the QUESTION.

\prompthead{Control semantics and decision policy}
\begin{promptrules}
\item \texttt{delegate\_task} starts a NEW SubAgent attempt. Give it a specific, actionable subtask, relevant context, an available model, and the tools needed to execute it.
\item \texttt{revise\_task} updates the CURRENT active SubAgent and keeps the same attempt alive. Inspect its state, requested adjustment, and supporting evidence before deciding.
\item When revising, determine whether to keep, reduce, expand, remove, or decompose the task, and whether context, tools, budget, constraints, or state handoff are needed. Provide a concrete adjustment or corrected solution path.
\item Attempts are limited. After a genuine \texttt{stop}, use evidence from the abandoned path to start a new attempt. Revision requests normally call for adjusting the active worker.
\item Use working memory as the compact interaction state. When verified evidence is sufficient, use \texttt{complete} with a concise final answer to the original question.
\end{promptrules}

\prompthead{Input context}
\promptfield{Question}{\promptslot{original question and attachment information}}
\promptfield{Progress}{\promptslot{attempt index; remaining attempts; active SubAgent}}
\promptfield{Working memory}{\promptslot{prior structured communication reports}}
\promptfield{Resources}{\promptslot{available models; tool descriptions and schemas}}

\prompthint
\prompthead{Output}
Return ONLY one JSON object. Use one of the following action formats.
\begin{promptcode}
\pc{\{"action": "delegate\_task", "reasoning": "\textless{}why more work\textgreater{}",}
\pc{\hspace*{0.525em}"params": \{}
\pc{\hspace*{1.5750000000000002em}"task\_instruction": "\textless{}specific, actionable subtask\textgreater{}",}
\pc{\hspace*{1.5750000000000002em}"context": "\textless{}relevant prior findings and evidence\textgreater{}",}
\pc{\hspace*{1.5750000000000002em}"model": "\textless{}available model\textgreater{}", "tools": ["\textless{}tool\_name\textgreater{}"]}
\pc{\hspace*{0.525em}\}\}}
\end{promptcode}
\begin{promptcode}
\pc{\{"action": "revise\_task", "reasoning": "\textless{}why revise\textgreater{}",}
\pc{\hspace*{0.525em}"params": \{}
\pc{\hspace*{1.5750000000000002em}"task\_scope": \{}
\pc{\hspace*{2.625em}"operation": "keep|reduce|expand|remove|decompose",}
\pc{\hspace*{2.625em}"object": "\textless{}task object\textgreater{}", "details": "\textless{}scope change\textgreater{}"\},}
\pc{\hspace*{1.5750000000000002em}"resource": \{}
\pc{\hspace*{2.625em}"operation": "keep|request",}
\pc{\hspace*{2.625em}"object": "context|tool|budget|constraint|state\_handoff|none",}
\pc{\hspace*{2.625em}"details": "\textless{}resource adjustment\textgreater{}"\},}
\pc{\hspace*{1.5750000000000002em}"context": "\textless{}additional findings or guidance\textgreater{}",}
\pc{\hspace*{1.5750000000000002em}"tools": ["\textless{}tool\_name\textgreater{}"], "step\_budget": 10,}
\pc{\hspace*{1.5750000000000002em}"instruction": "\textless{}revised instruction, if needed\textgreater{}"}
\pc{\hspace*{0.525em}\}\}}
\end{promptcode}
\begin{promptcode}
\pc{\{"action": "complete", "reasoning": "\textless{}supporting evidence\textgreater{}",}
\pc{\hspace*{0.525em}"params": \{"answer": "\textless{}concise final answer\textgreater{}"\}\}}
\end{promptcode}
Here, \texttt{step\_budget} supplies additional steps to the active attempt.
\end{coreprompt}

\clearpage
\begin{coreprompt}{GAIA2: Main Agent}{myblue}
You are the MainAgent in \ourmethod{} for GAIA2. Solve the current turn by coordinating one active SubAgent at a time. GAIA2 is a stateful application environment: correctness depends on exact side effects, including messages, calendar events, files, and user-facing responses.

\prompthead{Decomposition principles}
\begin{promptrules}
\item Delegate local, verifiable subtasks with a clear objective, temporal scope, permitted writes, and stopping condition. Assign the minimal tools needed.
\item If the execution path depends on unknown observations, specify the objective and constraints; let the SubAgent determine the concrete steps from tool results.
\end{promptrules}
\prompthead{Important rules}
\begin{promptrules}
\item Do not call application tools directly. Use \texttt{delegate\_task} to start a worker and \texttt{revise\_task} to adjust the active worker while preserving its progress.
\item The original task is the persistent contract; later notifications supplement it. Preserve exact entities, times, conditions, and required message content in assignments.
\item Avoid extra side effects. Create only required files; keep internal state in working memory.
\item Assign \texttt{SystemApp\_\_wait\_for\_notification} when waiting is needed and \texttt{AgentUserInterface\_\_send\_message\_to\_user} for required user reports. Do not guess future events; delegate observation and waiting.
\item Treat \texttt{follow\_up} reports as requests to synchronize external information. Inspect the update and respond through \texttt{revise\_task}.
\item Use \texttt{finalize\_task} only after all required work and side effects are complete. Its message is sent to the user by the runner.
\end{promptrules}

\prompthead{Input context}
\promptfield{Original task}{\promptslot{persistent user requirements}}
\promptfield{Updates}{\promptslot{notification history; reported follow-up events}}
\promptfield{State}{\promptslot{active SubAgent; working memory; remaining budget}}
\promptfield{Tools}{\promptslot{available application tools and argument schemas}}

\prompthint
\prompthead{Output}
Return ONLY one JSON object using one of these action formats.
\begin{promptcode}
\pc{\{"action": "delegate\_task", "reasoning": "\textless{}why delegate\textgreater{}",}
\pc{\hspace*{0.525em}"params": \{}
\pc{\hspace*{1.5750000000000002em}"task\_instruction": "\textless{}executable subtask and stopping condition\textgreater{}",}
\pc{\hspace*{1.5750000000000002em}"context": "\textless{}findings, permitted writes, and constraints\textgreater{}",}
\pc{\hspace*{1.5750000000000002em}"tools": ["\textless{}App\_\_tool\_name\textgreater{}"], "step\_budget": 30\}\}}
\end{promptcode}
\begin{promptcode}
\pc{\{"action": "revise\_task", "reasoning": "\textless{}why continue\textgreater{}",}
\pc{\hspace*{0.525em}"params": \{}
\pc{\hspace*{1.5750000000000002em}"task\_scope": \{}
\pc{\hspace*{2.625em}"operation": "keep|reduce|expand|remove|decompose|follow\_up",}
\pc{\hspace*{2.625em}"object": "\textless{}task object\textgreater{}", "details": "\textless{}scope change\textgreater{}"\},}
\pc{\hspace*{1.5750000000000002em}"resource": \{}
\pc{\hspace*{2.625em}"operation": "keep|request",}
\pc{\hspace*{2.625em}"object": "context|tool|budget|constraint|state\_handoff|none",}
\pc{\hspace*{2.625em}"details": "\textless{}resource adjustment\textgreater{}"\},}
\pc{\hspace*{1.5750000000000002em}"context": "\textless{}guidance\textgreater{}", "tools": ["\textless{}App\_\_tool\_name\textgreater{}"],}
\pc{\hspace*{1.5750000000000002em}"step\_budget": 10, "instruction": "\textless{}revised subtask\textgreater{}"\}\}}
\end{promptcode}
\begin{promptcode}
\pc{\{"action": "finalize\_task", "reasoning": "\textless{}all requirements met\textgreater{}",}
\pc{\hspace*{0.525em}"params": \{"message": "\textless{}final user-facing message\textgreater{}"\}\}}
\end{promptcode}
\end{coreprompt}

\clearpage
\begin{coreprompt}{SWE-Bench: Main Agent}{myblue}
You are the MainAgent for \ourmethod{} on SWE-Bench. Use execution evidence to manage SubAgents solving the GitHub issue. Only SubAgents inspect files, edit code, run commands and tests.

\prompthead{Control semantics}
\begin{promptrules}
\item \texttt{delegate\_task} starts a NEW attempt in a FRESH container, discarding unsubmitted code and local state from the previous attempt.
\item \texttt{revise\_task} preserves the CURRENT container and code changes. With \texttt{executor.operation=keep}, continue the same worker and memory; with \texttt{replace}, start a fresh worker context over the same workspace.
\item \texttt{submit} runs the official grader on the current container.
\end{promptrules}
\prompthead{Decision policy}
\begin{promptrules}
\item Initially delegate the issue as an actionable engineering task. On a \texttt{revise} report, inspect the request and evidence; normally keep the same worker.
\item Submit a completed patch supported by relevant tests. Do not replace the worker merely for optional review, alternative implementations, or repeating passing tests.
\item All attempts and revisions share the global step budget. Do not request more steps than remain; when no steps remain, submit the current container.
\item Historical issue discussions do not override the latest evidence-backed report unless the worker identifies a concrete unresolved problem.
\end{promptrules}

\prompthead{Input context}
\promptfield{Issue}{\promptslot{GitHub issue; supplied hints; repository metadata}}
\promptfield{Progress}{\promptslot{remaining attempts; used and remaining global steps}}
\promptfield{State}{\promptslot{active worker; working memory and communication evidence}}
\promptfield{Models}{\promptslot{available execution models}}

\prompthead{Output}
Return ONLY one JSON object. No prose, Markdown, or code fences.
\begin{promptcode}
\pc{\{"action": "delegate\_task", "reasoning": "\textless{}why a fresh attempt\textgreater{}",}
\pc{\hspace*{0.525em}"params": \{}
\pc{\hspace*{1.5750000000000002em}"task\_instruction": "\textless{}actionable software-engineering task\textgreater{}",}
\pc{\hspace*{1.5750000000000002em}"context": "\textless{}relevant context and prior evidence\textgreater{}",}
\pc{\hspace*{1.5750000000000002em}"model": "\textless{}available model\textgreater{}", "tools": []\}\}}
\end{promptcode}
\begin{promptcode}
\pc{\{"action": "revise\_task", "reasoning": "\textless{}why this update\textgreater{}",}
\pc{\hspace*{0.525em}"params": \{}
\pc{\hspace*{1.5750000000000002em}"executor": \{"operation": "keep|replace",}
\pc{\hspace*{8.4em}"model": "\textless{}optional model for replacement\textgreater{}"\},}
\pc{\hspace*{1.5750000000000002em}"task\_scope": \{}
\pc{\hspace*{2.625em}"operation": "keep|reduce|expand|remove|decompose",}
\pc{\hspace*{2.625em}"object": "\textless{}task object\textgreater{}", "details": "\textless{}scope change\textgreater{}"\},}
\pc{\hspace*{1.5750000000000002em}"resource": \{}
\pc{\hspace*{2.625em}"operation": "keep|request",}
\pc{\hspace*{2.625em}"object": "context|tool|budget|constraint|state\_handoff|none",}
\pc{\hspace*{2.625em}"details": "\textless{}resource adjustment\textgreater{}"\},}
\pc{\hspace*{1.5750000000000002em}"context": "\textless{}guidance, or complete handoff when replacing\textgreater{}",}
\pc{\hspace*{1.5750000000000002em}"tools": [], "step\_budget": 10,}
\pc{\hspace*{1.5750000000000002em}"instruction": "\textless{}revised task; preserve inherited patch\textgreater{}"\}\}}
\end{promptcode}
\begin{promptcode}
\pc{\{"action": "submit", "reasoning": "\textless{}readiness evidence\textgreater{}",}
\pc{\hspace*{0.525em}"params": \{"reason": "\textless{}reason to submit\textgreater{}"\}\}}
\end{promptcode}
\end{coreprompt}

\clearpage
\subsection{Sub-Agent Prompts}
\label{app:sub-agent-prompt}
\label{app:communication-schema}
The status, request, and evidence signals are encoded in the \texttt{state}, \texttt{action}, and \texttt{evidence} fields, respectively. The \texttt{action} field also specifies the communication action type. The implementation action \texttt{finish} corresponds to the sub-agent \texttt{submit} action in our formulation. The templates below retain the benchmark-specific output wrappers.

\begin{coreprompt}{GAIA: Sub-Agent}{mypurple}
You are a specialized SubAgent in \ourmethod{}. Use the available tools to complete your task.

\prompthead{Assignment and execution context}
\promptfield{Progress}{\promptslot{current step; remaining steps; budget warning}}
\promptfield{Your task}{\promptslot{subtask assigned by MainAgent}}
\promptfield{Context}{\promptslot{MainAgent context; original question}}
\promptfield{Available tools}{\promptslot{assigned tools and argument schemas}}
\promptfield{Local state}{\promptslot{memory; current tool observation}}

\prompthead{Execution and communication rules}
\begin{promptrules}
\item Focus on the assigned task; address the whole question only when necessary. Use tools until you can make a useful handoff or need a task or resource adjustment.
\item Use \texttt{finish} when the assigned subtask is complete; use \texttt{revise} when changing the task or resources would enable progress; use \texttt{stop} when the current path should be abandoned.
\item Request additional budget before exhaustion when continuing the current path is useful. Explain the needed adjustment and why it enables progress.
\item \texttt{state} contains facts: what is found, completed, ruled out, or unresolved. Put requests in \texttt{action.request} and your analysis in \texttt{state} or \texttt{action.reason}.
\item \texttt{evidence.summary} must quote relevant raw tool outputs verbatim, without paraphrasing or adding conclusions. If none exists, write ``No tool output available.''
\item The inner \texttt{action.type} must match the top-level action. Use \texttt{resource.object=none} when no resource change is needed.
\end{promptrules}

\prompthead{Ordinary tool output}
Return ONLY JSON, with one tool action and a brief memory update.
\begin{promptcode}
\pc{\{}
\pc{\hspace*{1.05em}"action": "\textless{}tool\_name\textgreater{}",}
\pc{\hspace*{1.05em}"params": \{"\textless{}argument\textgreater{}": "\textless{}value\textgreater{}"\},}
\pc{\hspace*{1.05em}"memory": "\textless{}key observation\textgreater{}"}
\pc{\}}
\end{promptcode}
\prompthead{Structured communication output}
For \texttt{finish}, \texttt{revise}, or \texttt{stop}, return ONLY JSON:
\begin{promptcode}
\pc{\{}
\pc{\hspace*{1.05em}"action": "finish|revise|stop",}
\pc{\hspace*{1.05em}"params": \{}
\pc{\hspace*{2.1em}"state": "\textless{}concise factual local state\textgreater{}",}
\pc{\hspace*{2.1em}"action": \{}
\pc{\hspace*{3.1500000000000004em}"type": "finish|revise|stop",}
\pc{\hspace*{3.1500000000000004em}"task\_scope": \{}
\pc{\hspace*{4.2em}"operation": "keep|reduce|expand|remove|decompose",}
\pc{\hspace*{4.2em}"object": "\textless{}task object\textgreater{}", "details": "\textless{}scope change\textgreater{}"\},}
\pc{\hspace*{3.1500000000000004em}"resource": \{}
\pc{\hspace*{4.2em}"operation": "keep|request",}
\pc{\hspace*{4.2em}"object": "context|tool|budget|constraint|state\_handoff|none",}
\pc{\hspace*{4.2em}"details": "\textless{}needed resource or no change\textgreater{}"\},}
\pc{\hspace*{3.1500000000000004em}"request": "\textless{}concrete request to MainAgent\textgreater{}",}
\pc{\hspace*{3.1500000000000004em}"reason": "\textless{}why this request is justified\textgreater{}"}
\pc{\hspace*{2.1em}\},}
\pc{\hspace*{2.1em}"evidence": \{"summary": "\textless{}verbatim relevant tool output\textgreater{}"\}}
\pc{\hspace*{1.05em}\},}
\pc{\hspace*{1.05em}"memory": "\textless{}brief local memory\textgreater{}"}
\pc{\}}
\end{promptcode}
\end{coreprompt}

\clearpage
\begin{coreprompt}{GAIA2: Sub-Agent}{mypurple}
You are a specialized SubAgent in \ourmethod{} for GAIA2. Complete only the assigned subtask using the available application tools. Their side effects count for evaluation; do not make extra changes.

\prompthead{Execution and external updates}
\begin{promptrules}
\item If context, tools, constraints, or budget are insufficient, request a revision from MainAgent. Send exact message content supplied by MainAgent verbatim.
\item When a new user or environment notification arrives, report it to MainAgent before further side effects. Use \texttt{revise} with \texttt{task\_scope.operation=follow\_up}.
\item Retain new notifications and relevant prior observations in \texttt{evidence.raw}. Copy IDs, timestamps, message content, and other reusable values exactly, retaining only relevant details.
\end{promptrules}
\prompthead{User communication}
\begin{promptrules}
\item Use verified observations to fill in communication details, preserving relevant entities, dates, locations, conditions, and user intent.
\item \texttt{AgentUserInterface\_\_send\_message\_to\_user} ends an ARE turn. Complete currently authorized side effects first, unless a user reply is required to proceed.
\end{promptrules}
\prompthead{Control actions}
Use \texttt{finish} for a completed subtask or useful handoff; \texttt{revise} for an adjustment to task scope, context, tools, constraints, or budget; and \texttt{stop} for an invalid or abandoned path. Request useful extra steps before exhaustion, with \texttt{resource.object=budget} and a concrete explanation.

\prompthead{Input context}
\promptfield{Assignment}{\promptslot{subtask; MainAgent context; original user task}}
\promptfield{Execution state}{\promptslot{budget; history; observations; notifications}}
\promptfield{Tools}{\promptslot{assigned application APIs and argument schemas}}

\prompthead{Ordinary tool call}
Use exactly one tool call with the following wrapper:
\begin{promptcode}
\pc{Thought: \textless{}purpose of the next action\textgreater{}}
\pc{Action:}
\pc{\{"action": "\textless{}App\_\_tool\_name\textgreater{}",}
\pc{\hspace*{0.525em}"action\_input": \{"\textless{}arg\textgreater{}": "\textless{}value\textgreater{}"\}\}}
\pc{\textless{}end\_action\textgreater{}}
\end{promptcode}
\prompthead{Structured communication control action}
\begin{promptcode}
\pc{Thought: \textless{}current judgment and reason for the handoff\textgreater{}}
\pc{Action:}
\pc{\{}
\pc{\hspace*{1.05em}"action": "finish|revise|stop",}
\pc{\hspace*{1.05em}"action\_input": \{}
\pc{\hspace*{2.1em}"state": "\textless{}concise factual local state\textgreater{}",}
\pc{\hspace*{2.1em}"action": \{}
\pc{\hspace*{3.1500000000000004em}"type": "finish|revise|stop",}
\pc{\hspace*{3.1500000000000004em}"task\_scope": \{}
\pc{\hspace*{4.2em}"operation": "keep|reduce|expand|remove|decompose|follow\_up",}
\pc{\hspace*{4.2em}"object": "\textless{}task object\textgreater{}", "details": "\textless{}scope change\textgreater{}"\},}
\pc{\hspace*{3.1500000000000004em}"resource": \{}
\pc{\hspace*{4.2em}"operation": "keep|request",}
\pc{\hspace*{4.2em}"object": "context|tool|budget|constraint|state\_handoff|none",}
\pc{\hspace*{4.2em}"details": "\textless{}needed resource or no change\textgreater{}"\},}
\pc{\hspace*{3.1500000000000004em}"request": "\textless{}request\textgreater{}", "reason": "\textless{}supporting reason\textgreater{}"}
\pc{\hspace*{2.1em}\},}
\pc{\hspace*{2.1em}"evidence": \{}
\pc{\hspace*{3.1500000000000004em}"summary": "\textless{}key observation supporting this report\textgreater{}",}
\pc{\hspace*{3.1500000000000004em}"raw": \{"\textless{}name\textgreater{}": "\textless{}exact tool input/output value\textgreater{}"\}\}}
\pc{\hspace*{1.05em}\}}
\pc{\}\textless{}end\_action\textgreater{}}
\end{promptcode}
\end{coreprompt}

\clearpage
\begin{coreprompt}{SWE-Bench: Sub-Agent}{mypurple}
You are a \ourmethod{} SubAgent responsible for a local software-engineering task. Complete the GitHub issue subtask assigned by MainAgent. The repository is already checked out in the task container.

\prompthead{Assignment and execution context}
\promptfield{Issue}{\promptslot{original GitHub issue and supplied hints}}
\promptfield{Execution contract}{\promptslot{assigned subtask; MainAgent context}}
\promptfield{Progress}{\promptslot{current step; remaining steps; budget warning}}
\promptfield{Environment}{\promptslot{repository state; assigned ACI/Bash interface}}
\promptfield{Local state}{\promptslot{memory; current command observation}}

\prompthead{Execution and communication policy}
\begin{promptrules}
\item Inspect the repository and existing diff, locate relevant code, implement the fix, and run relevant tests. Preserve useful inherited work and ground progress claims in observed command and test results.
\item Use ordinary ACI or Bash commands while the task and resources are sufficient. Handle ordinary command errors and test failures locally before requesting revision.
\item Use \texttt{finish} when the assigned task is complete and the patch is ready for MainAgent to submit, with relevant test evidence in the handoff.
\item Use \texttt{revise} when MainAgent must adjust task scope, context, tools, constraints, budget, or state handoff. If promising work needs more steps, request them before exhaustion with \texttt{resource.object=budget}.
\item Use \texttt{stop} when the current path should be abandoned. Use \texttt{revise} for incomplete work whose patch or observations are worth preserving.
\item MainAgent handles final submission and evaluation. Never run the official submission command.
\end{promptrules}

\prompthead{Environment action format}
Execute one real ACI or Bash command through the assigned interface. For the text interface, return:
\begin{promptcode}
\pc{DISCUSSION}
\pc{\textless{}concise judgment and purpose of the next command\textgreater{}}
\pc{COMMAND}
\pc{\textless{}one ACI or Bash command\textgreater{}}
\end{promptcode}
For native tools, call \texttt{aci\_command} with \texttt{discussion} and \texttt{command}. Place ACI keywords such as \texttt{open} and \texttt{search\_file} inside \texttt{command}; they are not native tools.

\prompthead{Structured communication control format}
Return the following textual control action under either interface:
\begin{promptcode}
\pc{DISCUSSION}
\pc{\textless{}concise judgment and reason for handing control to MainAgent\textgreater{}}
\pc{COMMAND}
\pc{\textless{}finish|revise|stop\textgreater{} \textless{}single-line structured communication JSON\textgreater{}}
\end{promptcode}
The structured communication JSON has the following fields (expanded here for readability):
\begin{promptcode}
\pc{\{}
\pc{\hspace*{1.05em}"state": "\textless{}completed work, remaining issue, and patch status\textgreater{}",}
\pc{\hspace*{1.05em}"action": \{}
\pc{\hspace*{2.1em}"type": "finish|revise|stop",}
\pc{\hspace*{2.1em}"task\_scope": \{}
\pc{\hspace*{3.1500000000000004em}"operation": "keep|reduce|expand|remove|decompose",}
\pc{\hspace*{3.1500000000000004em}"object": "\textless{}task object\textgreater{}", "details": "\textless{}scope change\textgreater{}"\},}
\pc{\hspace*{2.1em}"resource": \{}
\pc{\hspace*{3.1500000000000004em}"operation": "keep|request",}
\pc{\hspace*{3.1500000000000004em}"object": "context|tool|budget|constraint|state\_handoff|none",}
\pc{\hspace*{3.1500000000000004em}"details": "\textless{}needed resource or no change\textgreater{}"\},}
\pc{\hspace*{2.1em}"request": "\textless{}request\textgreater{}", "reason": "\textless{}supporting reason\textgreater{}"}
\pc{\hspace*{1.05em}\},}
\pc{\hspace*{1.05em}"evidence": \{"summary": "\textless{}observed changes and test outcomes\textgreater{}"\}}
\pc{\}}
\end{promptcode}
Choose one environment or control action per turn. Keep control JSON on one line, with \texttt{action.type} matching the command. Add no text after the command or control payload.
\end{coreprompt}

\clearpage
\subsection{Generative Reward Model Prompt}
\label{app:attribution-prompt}
The implementation fields \texttt{needs\_repair} and \texttt{hint} correspond to $\texttt{FLAG}_t$ and $\mathcal H_t$ in the paper. The hint enters the GAIA and GAIA2 main-agent templates at the highlighted position.

\begin{coreprompt}{Attribution: Evaluating an Orchestration Decision}{myblue}
You are a process reward model that extracts hindsight hints for MainAgent orchestration decisions in a mixed-initiative multi-agent system.

The MainAgent creates a SubAgent by defining its subtask, supplying context, and configuring resources. During execution, the SubAgent may proactively request changes to its task scope, context, tool permissions, or execution budget. MainAgent responds by revising the worker's task contract.

\prompthead{You will receive}
\begin{promptrules}
\item The context available to MainAgent when the target decision was made.
\item The target decision: a delegation creating a SubAgent, or a revision responding to its request.
\item The subsequent SubAgent execution trace, including later revision requests, MainAgent responses, and tool or environment results.
\end{promptrules}
Determine whether the later requests and trace reveal concrete, useful hindsight information that could have helped MainAgent improve the target decision.

\prompthead{Attribution rules}
\begin{promptrules}
\item Evaluate only the specified target decision. Generate a hint only when the later evidence reveals a problem attributable to that decision.
\item Attribute only choices MainAgent could reasonably make from its decision-time context. Do not assume that entities or facts discovered later were already known.
\item A follow-up revision request does not by itself establish that the earlier decision was defective. Likewise, a task failure, timeout, or tool error alone does not identify an orchestration mistake; ground attribution in the decision and the observed execution.
\item For revisions, assess whether MainAgent understood the request, resolved the key blocker, adjusted task scope or resources appropriately, and avoided new execution problems.
\item If the worker proceeds successfully as intended and the trace reveals no concrete problem caused by the target decision, choose false.
\item Do not overthink. If attribution is uncertain or no concrete improvement is supported, choose false.
\end{promptrules}

\prompthead{Improvement hint}
When a repair is supported, give one concrete, actionable improvement to the target decision. Specify what MainAgent should change in its instruction, context, or resource configuration and connect it to the observed problem. Do not merely restate that execution failed or provide the final task answer.

\prompthead{Input context}
\promptfield{Decision-time context}{\promptslot{MainAgent input before the action}}
\promptfield{Target decision}{\promptslot{delegation or revision to evaluate}}
\promptfield{Subsequent trace}{\promptslot{worker actions and observations; proactive reports; MainAgent responses}}

\prompthead{Required output}
Return ONLY JSON, without Markdown or explanation. When a concrete repair is supported:
\begin{promptcode}
\pc{\{}
\pc{\hspace*{1.05em}"needs\_repair": true,}
\pc{\hspace*{1.05em}"hint": "\textless{}one concrete improvement to the target decision\textgreater{}"}
\pc{\}}
\end{promptcode}
Otherwise:
\begin{promptcode}
\pc{\{}
\pc{\hspace*{1.05em}"needs\_repair": false,}
\pc{\hspace*{1.05em}"hint": ""}
\pc{\}}
\end{promptcode}
\end{coreprompt}

\clearpage
\endgroup